\documentclass[sigconf, screen]{acmart}
\AtBeginDocument{%
  }

\copyrightyear{2025} 
\acmYear{2025} 
\setcopyright{acmlicensed}
\acmConference[MM '25]{Proceedings of the 33rd ACM International Conference on Multimedia}{October 27--31, 2025}{Dublin, Ireland}
\acmBooktitle{Proceedings of the 33rd ACM International Conference on Multimedia (MM '25), October 27--31, 2025, Dublin, Ireland}
\acmDOI{10.1145/3746027.3754730}
\acmISBN{979-8-4007-2035-2/2025/10}

\usepackage{makecell}
\usepackage{color}

\DeclareMathOperator*{\argmax}{arg\,max}

\def\balpha{\mbox{{\boldmath $\alpha$}}}
\def\bbeta{\mbox{{\boldmath $\beta$}}}

\def\bdelta{\mbox{{\boldmath $\delta$}}}

\def\beps{\mbox{{\boldmath $\epsilon$}}}
\def\bepsilon{\mbox{{\boldmath $\epsilon$}}}

\def\bnu{\mbox{{\boldmath $\nu$}}}

\def\bxi{\mbox{{\boldmath $\xi$}}}

\def\bdelta{\mbox{{\boldmath $\delta$}}}

\def\mL{{\mathcal L}}

\def\mP{{\mathcal P}}

\def\mT{{\mathcal T}}

\DeclareMathAlphabet\mathbfcal{OMS}{cmsy}{b}{n}

\def\0{{\bf 0}}
\def\1{{\bf 1}}

\def\bW{{\bf W}}

\def\bx{{\bf x}}
\def\by{{\bf y}}

\def\mmE{{\mathbb E}}

\def\mmR{{\mathbb R}}

\def\bx{{\bf x}}

\def\by{{\bf y}}

\def\bW{{\bf W}}
\def\balpha{{\bm \alpha}}
\def\bbeta{{\bm \beta}}

\def\eg{\emph{e.g.}} 
\def\ie{\emph{i.e.}}

\newtheorem{prop}{Proposition}

\newtheorem{ass}{Assumption}

\usepackage{mathtools}
\usepackage{bm}
\usepackage{pifont}
\usepackage{colortbl} 
\usepackage{microtype}
\usepackage{multirow}
\usepackage[ruled,linesnumbered]{algorithm2e}
\usepackage{enumitem}
\usepackage{balance}

\AtEndPreamble{
    \usepackage[capitalize]{cleveref}
    \crefname{section}{Sec.}{Secs.}
    \Crefname{section}{Section}{Sections}
    \Crefname{table}{Table}{Tables}
    \crefname{table}{Tab.}{Tabs.}
}

\newcommand{\ourname}{ZOA}
\title{Test-Time Model Adaptation for Quantized Neural Networks}

\author{Zeshuai Deng}
\authornote{Equal contribution.}
\affiliation{
  \institution{South China University of Technology}
  \city{Guangzhou}
  \country{China}
}
\email{sedengzeshuai@mail.scut.edu.cn}

\author{Guohao Chen}
\authornotemark[1]
\affiliation{
  \institution{Nanyang Technological University}
  \city{Singapore}
  \country{Singapore}
}
\email{guohao.chen@ntu.edu.sg}

\author{Shuaicheng Niu}
\authornotemark[1]
\affiliation{
  \institution{Nanyang Technological University}
  \city{Singapore}
  \country{Singapore}
}
\email{shuaicheng.niu@ntu.edu.sg}

\author{Hui Luo}
\affiliation{
  \institution{Institute of Optics and Electronics, Chinese Academy of Sciences}
  \city{Chengdu}
  \country{China}
}
\email{luohui19@mails.ucas.ac.cn}

\author{Shuhai Zhang}
\affiliation{
  \institution{South China University of Technology}
  \city{Guangzhou}
  \country{China}
}
\email{shuhaizhangshz@gmail.com}

\author{Yifan Yang}
\affiliation{
  \institution{South China University of Technology}
  \city{Guangzhou}
  \country{China}
}
\email{youngyif1@gmail.com}

\author{Renjie Chen}
\affiliation{
  \institution{South China University of Technology}
  \city{Guangzhou}
  \country{China}
}
\email{202410190283@mail.scut.edu.cn}

\author{Wei Luo}
\authornotemark[1]
\affiliation{
  \institution{South China Agricultural University}
  \city{Guangzhou}
  \country{China}
}
\email{cswluo@scau.edu.cn}

\author{Mingkui Tan}
\authornote{Corresponding author.}
\affiliation{
  \institution{South China University of Technology}
  \city{Guangzhou}
  \country{China}
}
\email{mingkuitan@scut.edu.cn}

\renewcommand{\shortauthors}{Zeshuai Deng et al.}

\begin{document}

\begin{abstract}
Quantizing deep models prior to deployment is a widely adopted technique to speed up inference for various real-time applications, such as autonomous driving. However, quantized models often suffer from severe performance degradation in dynamic environments with potential domain shifts and this degradation is significantly more pronounced compared with their full-precision counterparts, as shown by our theoretical and empirical illustrations.
To address the domain shift problem, test-time adaptation (TTA) has emerged as an effective solution by enabling models to learn adaptively from test data. 
Unfortunately, existing TTA methods are often impractical for quantized models as they typically rely on gradient backpropagation—an operation that is unsupported on quantized models due to vanishing gradients, as well as memory and latency constraints.
In this paper, we focus on TTA for quantized models to improve their robustness and generalization ability efficiently. We propose a continual zeroth-order adaptation (ZOA) framework that enables efficient model adaptation using only two forward passes, eliminating the computational burden of existing methods.
Moreover, we propose a domain knowledge management scheme to store and reuse different domain knowledge with negligible memory consumption, reducing the interference of different domain knowledge and fostering the knowledge accumulation during long-term adaptation.
Experimental results on three classical architectures, including quantized transformer-based and CNN-based models, demonstrate the superiority of our methods for quantized model adaptation. On the quantized W6A6 ViT-B model, our ZOA is able to achieve a 5.0$\%$ improvement over the state-of-the-art FOA on ImageNet-C dataset. The source code is available at \href{https://github.com/DengZeshuai/ZOA.git}{https://github.com/DengZeshuai/ZOA}.
\end{abstract}



\maketitle

\begin{figure}
\centering
\includegraphics[width=1.0\linewidth]{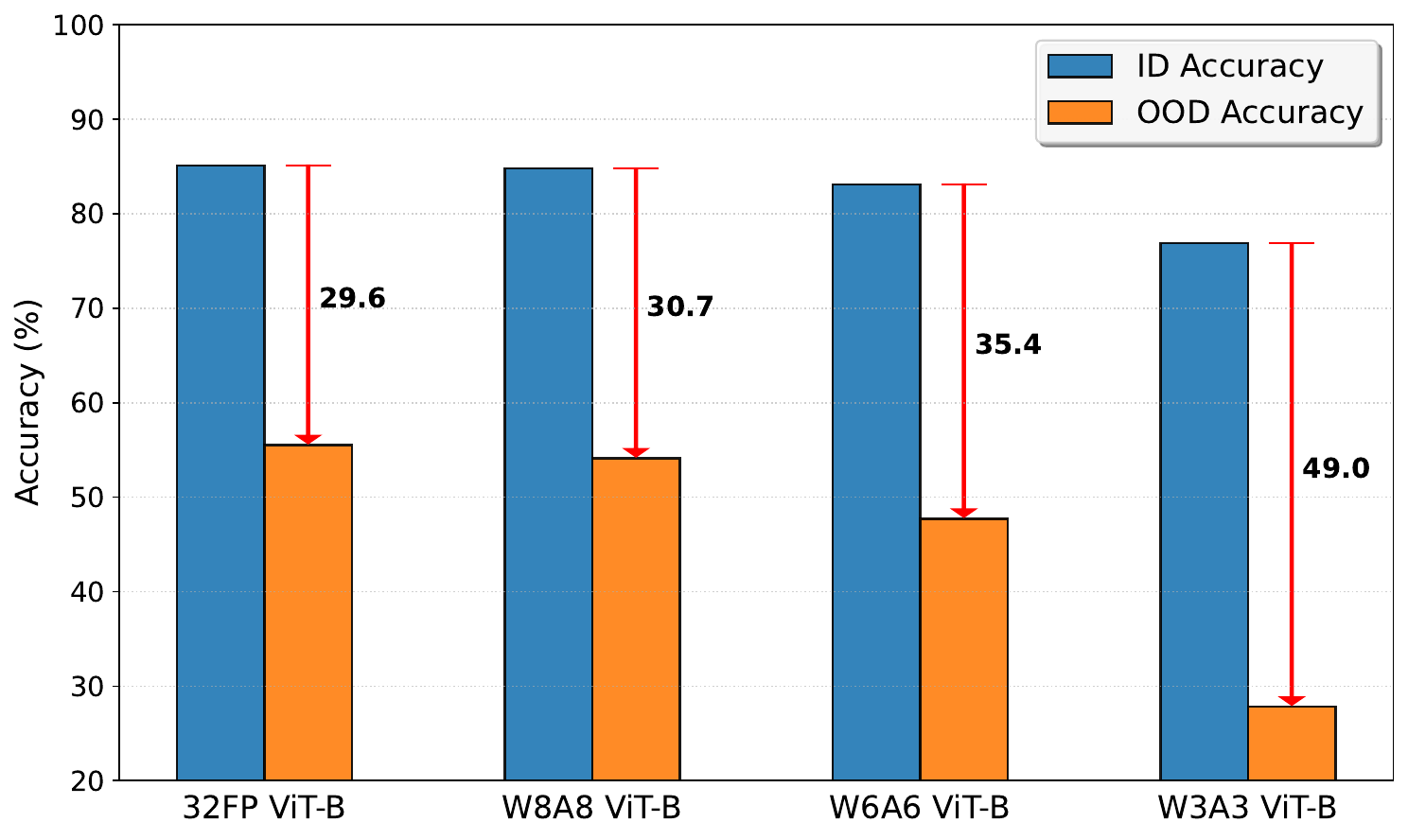}
\caption{
Quantized models often exhibit substantial performance degradation on out-of-distribution (OOD) data compared to their full-precision counterparts.
The {red arrows} represent the accuracy gap between the results of models on in-distribution (ID) data (ImageNet) and OOD data (ImageNet-C).
``32FP'' denotes the 32-bit floating-point ViT-B model, and ``WNAN'' indicates the N-bit quantized ViT-B variant.
}
\label{fig:motivation}
\end{figure}

\section{Introduction}
\label{sec:intro}
The common paradigm in deploying deep neural network models involves quantizing full-precision models into low-bit precision, \eg, 8-bit precision, before deployment using training-aware quantization~\cite{gong2019differentiable, esser2019learned, jung2019learning, li2020additive} or post-training quantization~\cite{yang2019quantization, nagel2020up, wei2022qdrop, wu2024adalog}.
This workflow has become a standard practice in multimedia applications, aiming to accelerate inference and improve response speed across a wide range of scenarios, including those involving low-power edge devices or latency-sensitive contexts. However, after deployment, the environments may dynamically change with domain shifts, such as changes in lighting variations and sensor noise~\cite{wang2021tent, niu2022efficient, wang2022continual}, and under such domain shifts, these models may suffer from severe performance degradation.

In this paper, we identify that quantized models are more sensitive to domain shifts compared with their full-precision counterparts, exhibiting significantly greater performance degradation when encountering such shifts. As illustrated in~\cref{fig:motivation}, the accuracy gap between in-distribution (ID) and out-of-distribution (OOD) data widens markedly as the bit-width of the quantized model decreases. This is because the quantization process inevitably introduces loss errors when facing out-of-distribution perturbations, and this loss increases as the bit-precision decreases, \textit{see} our analyses in \cref{sec:sensitivity}. However, existing quantization methods typically optimize the quantized models on source ID data, overlooking their performance on OOD ones~\cite{lin2018defensive, xiao2023robustmq}. Therefore, there remains an urgent demand to efficiently adapt quantized models to new environments.

To address the domain shift issue, test-time adaptation (TTA)~\cite{wang2021tent, niu2022efficient, zhang2022memo, niu2023towards, lee2024entropy} has emerged as a promising paradigm by adapting deployed models to out-of-distribution data during test time. Based on whether they rely on backpropagation (BP), existing methods can be broadly categorized into two groups: (i) BP-based methods, such as entropy minimization~\cite{wang2021tent, niu2022efficient}, prediction consistency maximization~\cite{zhang2022memo}, and feature distribution alignment~\cite{mirza2023actmad}; and (ii) BP-free methods like calibrating the statistics of batch normalization layers ~\cite{nado2020evaluating, schneider2020improving} and correcting the predicted logits~\cite{iwasawa2021test}. However, effectively supporting various quantized models remains challenging for these approaches, making them impractical for scenarios such as deployment on edge devices. The reasons are as follows.

\textit{BP-based TTA} methods typically~\cite{wang2021tent,niu2022efficient,tan2025uncertainty} focus on full-precision models and rely on BP for model adaptation, which requires the function of model inference to be continuous and differentiable.
However, the forward process of quantized models exhibits discrete characteristics and lacks differentiability due to their extremely low-bitwidth parameter representations (\eg, 4-bit or 8-bit precision).
Besides, gradient computation through BP entails substantial memory overhead for retaining intermediate activation values.
While the edge devices, such as FPGAs and smartphones, are often resource-limited, making it unaffordable to perform these computationally expensive BP-based methods.
Therefore, it is non-trivial to directly use these methods for efficient quantized model adaptation.

\textit{BP-free TTA} methods~\cite{iwasawa2021test, boudiaf2022parameter, niu2024testtime} adapt the deployed models using forward passes only. However, conventional methods like T3A~\cite{iwasawa2021test} and LAME~\cite{boudiaf2022parameter} are learning-free. Without updating the core parameters of models (such as parameters of normalization layers), these methods often show limited learning capabilities(see results in \cref{tab:long_term}). Recently, the learning-based method FOA~\cite{niu2024testtime} has been proposed, demonstrating much better performance compared with learning-free approaches. However, it often requires many forward passes per test sample to achieve satisfactory results, which may be impractical for real-time applications. Moreover, FOA is specifically designed for transformer architectures, limiting its applicability to a broader range of model types. Thus, achieving real-time, resource-efficient adaptation for quantized models in long-term, dynamically changing environments remains an open and challenging problem.

In this paper, we propose a zeroth-order test-time adaptation algorithm (\ourname) that efficiently adapts quantized models using only two forward passes. 
Specifically, we design a continual domain knowledge learning method that accumulates knowledge across different domains.
To reduce the interference between different domains, we design a domain knowledge management scheme to store and reuse the learned knowledge. 
With a set of learnable coefficients to aggregate the knowledge from different domains, our \ourname ~efficiently enhances the reliability of quantized models in long-term adaptation.
We evaluate our \ourname ~and existing methods across three classical architectures and different quantization bit-precision.
Experimental results demonstrate the effectiveness of our methods on quantized model adaptation.
Overall, the contribution of our work is summarized as follows:
\begin{itemize}[leftmargin=4mm]
\item 
We identify that the quantized models exhibited higher sensitivity to domain shift than their full-precision counterparts. 
To achieve reliable AI applications under dynamic conditions on edge devices, we design a continual zeroth-order adaptation framework, namely \ourname, to efficiently adapt the quantized models during testing using only two forward passes per sample.
\item We introduce a zeroth-order continual domain knowledge learning scheme. This method reduces cross-domain knowledge interference and enables the accumulation and reuse of historical adaptation knowledge for more effective forward-only TTA. 
We further propose a domain knowledge management strategy to ensure the computational and memory efficiency of our \ourname.
\item Experiments on three classical architectures across different bit-precisions demonstrate that our \ourname ~is able to efficiently adapt the quantized models by storing and reusing the learned domain knowledge to boost the long-term test-time adaptation. 
\end{itemize}

\section{Related Works}
\subsection{Test-Time Adaptation}
Test-time Adaptation (TTA) aims to adapt a pre-trained model to unlabeled test data to handle the domain shift between training and testing data. According to whether the model adaptation relies on backpropagation, existing methods can be categorized into:
1)~\textit{Backpropagation-based methods}: Early TTA methods optimize a source model via backpropagation using
an extra self-supervision task at testing,
such as rotation prediction~\cite{sun2020test}, contrastive learning~\cite{liu2021ttt++,bartler2022mt3}, reconstruction learning~\cite{gandelsman2022test,deng2023efficient}, etc. However, self-supervised tasks are essentially proxy tasks, which usually alter the training process, and this solution is not always feasible in practice. To address this issue, fully test-time adaptation methods have been proposed, which update models using unlabeled test data through unsupervised learning objectives, including entropy minimization~\cite{wang2021tent,niu2022efficient,tan2025uncertainty}, energy alignment~\cite{yuan2024tea,choiad2024aptive}, and prediction consistency maximization~\cite{fleuret2021test,zhang2022memo,wang2022continual,chen2022contrastive}, etc. Nonetheless, the aforementioned methods suffer from computational and memory inefficiency and limited application scenarios, \eg, can not be applied to quantized models, due to their reliance on backpropagation. \linebreak
2)~\textit{Backpropagation-free methods}: A common approach for BP-free TTA is to calibrate the statistics of batch normalization (BN) using the mean and variance computed over the test data~\cite{schneider2020improving,wang2021tent,mancini2018kitting}. Nevertheless, this method requires multiple test data to calculate statistics and assumes a balanced class distribution within a batch. To address this, later studies adopt data augmentations~\cite{zhang2022memo} for single-sample BN calibration or a class-wise sample bank~\cite{gong2022note} for the imbalanced data stream. However, without optimizing the core parameters of models, these methods often show limited online learning capabilities.
Recently, a forward-optimization adaptation (FOA)~\cite{niu2024testtime} has been proposed to optimize quantized models without requiring backpropagation. 
However, FOA requires a long adaptation time—up to 28 forward passes per sample, which is impractical for real-time applications. Instead, our \ourname ~seeks to use only two forward passes per sample to adapt quantized models continually, boosting TTA's practicality in various resource-limited contexts.

\subsection{Zeroth-Order Optimization}
Zeroth-order Optimization (ZO) leverages forward passes to estimate gradients without backpropagation. 
Recent advances in ZO have significantly expanded its applications in machine learning~\cite{liu2020primer} and natural language processing~\cite{malladi2023fine}, particularly in scenarios where gradient information is unavailable or impractical to obtain~\cite{deng2022rlprompt}. Motivated by the need for black-box adaptation in NLP, methods like BBT~\cite{sun2022black} and BBTv2~\cite{sun2022bbtv2} have employed evolutionary strategies such as CMA-ES~\cite{hansen2003reducing} to optimize proprietary models without access to gradients, while RLPrompt~\cite{deng2022rlprompt} leverages reinforcement learning for prompt tuning. However, these approaches often face challenges with high variance and instability~\cite{liu2020primer}, especially in vision tasks. To address these limitations, recent work has focused on improving gradient estimation accuracy through techniques like two-sided approximations~\cite{oh2023blackvip} and reducing variance in ZO fine-tuning of large language models by sparse parameter perturbations~\cite{liu2024sparse}, and increased batch sizes~\cite{jiang2024zo} and tensorized adapters~\cite{yang2024adazeta}. Additionally, efforts to scale up ZO optimization have included integrating historical data~\cite{cheng2021convergence}, and reusing intermediate features~\cite{chen2024deepzero} to enhance efficiency and convergence rates. 
These advancements highlight the growing potential of ZO methods in handling complex, large-scale machine learning problems while maintaining computational efficiency and adaptability.
In this paper, we incorporate the advantage of zeroth-order optimization to update the core parameters of the quantized neural networks for test-time adaptation.

\begin{figure*}
    \centering
    \includegraphics[width=1.0 \linewidth]{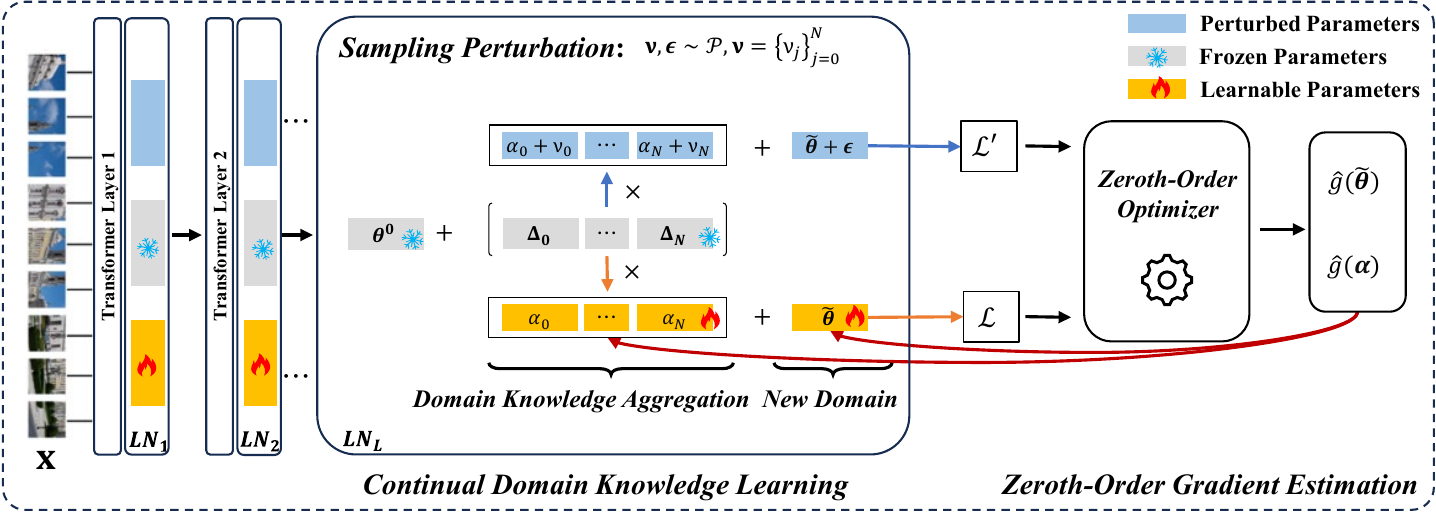}
    \caption{The overall framework of \ourname. We perturb the learnable parameters and the aggregation coefficients using the random perturbation vectors ${\bm \epsilon}$ and $\bnu$, which are sampled from a predefined distribution $\mP$.
    The gradient is estimated based on the loss $\mathcal{L}'$ and $\mathcal{L}$,
    {where $\mathcal{L}' = \mathcal{L}(\bx; \tilde{\bm \theta} + c {\bm \epsilon}, {\bm \alpha} + c {\bnu})$ is the loss related to perturbed parameters and $\mathcal{L} = \mathcal{L}(\bx; \tilde{\bm \theta}, {\bm \alpha})$ is related to learnable parameters.}
    The yellow blocks represent the learnable parameters, including $\balpha$ and $\tilde{\bm \theta}$. 
    The blue blocks represent the parameters perturbed by random vectors $\bm\nu$ and $\bm\epsilon$.
    The gray blocks represent the frozen parameters that would not be updated. 
    }
    \label{fig:overall}
\end{figure*}

\section{Motivation and Problem Statement}
In this section, we first theoretically demonstrate that QNNs are highly vulnerable to distribution shifts, and then revisit the limitations of existing TTA solutions in adapting quantized models.

\subsection{Sensitivity of QNN with Distribution Shift}
\label{sec:sensitivity}

We theoretically analyze the impact of out-of-distribution (OOD) perturbations on QNNs and provide justification for the necessity of TTA strategies in such scenarios, especially under low-bit precision.
 
\begin{prop} \label{prop:sensitivity}
Considering linear models, let $n$ denote the bit-precision of quantized models, for out-of-distribution (OOD) input perturbations $\delta$,  the quantization-induced loss difference $\Delta \mathcal{L} := \hat{\mathcal{L}}(\mathbf{x}+\delta) - \mathcal{L}(\mathbf{x}+\delta)$ satisfies:
\begin{equation} \label{eq:loss_sensitivity}
\Delta \mL >0, ~~ \text{and} ~~ \Delta \mL \propto \frac{1}{2^{2n}},
\end{equation}
where $\hat{\mathcal{L}}$ and $\mathcal{L}$ denote the MSE losses of quantized and full-precision models, respectively.
\end{prop}

The proof of Proposition \ref{prop:sensitivity} is provided in Appendix A. This result highlights two key insights: 1) Quantization always increases the loss under OOD perturbations, \ie, $\Delta \mL >0$;
and 2) This sensitivity grows exponentially as the bit-precision  $n$ decreases, \ie, $\Delta \mL \propto \frac{1}{2^{2n}}$.\linebreak
Our empirical observations in \cref{fig:motivation} are also aligned with our theoretical analysis. Overall, these findings underscore a critical limitation of QNNs and motivate the need for an effective TTA solution to enhance their robustness against distribution shifts.

\subsection{Test-Time Adaptation}
Let $f_{\bm \theta}(\cdot)$ be the model trained on the source training dataset $D_s = \{\bx_i, \by_i \}^{N_s}_{i=1}$ and $\bx_i \sim P(\bx)$.
During testing, test samples can be drawn from a shifted and dynamically changing distribution $Q(\bx)$, where $Q(\bx) \neq P(\bx)$, causing performance degradation in $f_{\bm \theta}(\cdot)$ predictions.
To adapt the pre-trained model $f_{\bm \theta}(\cdot)$ to the target domain, conventional TTA methods~\cite{wang2021tent, niu2022efficient} seek to optimize the model using some un/self-supervised learning objective on test samples:
\begin{equation} \label{eq:tta}
\min_{\Tilde{\bm \theta}} \mathcal{L}(\bx; {\bm \theta}), ~ \bx \sim Q(\bx),
\end{equation}
where $\Tilde{\bm \theta}\subseteq{\bm \theta}$ denotes the learnable parameters involved in TTA, and the test-time objective $\mathcal{L}(\cdot)$ is often formulated as entropy minimization~\cite{wang2021tent,niu2022efficient, wang2022continual}, activation alignment~\cite{mirza2023actmad,niu2024testtime}, and \textit{etc}.

Nevertheless, existing TTA methods typically require gradient-based updates through backpropagation, posing a fundamental limitation when applied to quantized models that do not support backpropagation. 
Recently, FOA~\cite{niu2024testtime} circumvents backpropagation by employing an evolutionary strategy to optimize input prompts. However, it 1) incurs substantial computational cost—requiring up to 28 forward passes per sample—and 2) assumes a low-dimensional adaptation space, making it less flexible for adapting deep quantized models with high-dimensional parameters.
Motivated by adaptation efficiency and applicability on various QNNs, in this paper, we seek to propose a novel TTA approach that effectively adapts a quantized model using only \textit{two forward passes} per sample—one forward pass for standard inference and a single extra forward pass for TTA—without assuming a low-dimensional adaptation space.

\begin{algorithm}[t] 
    \caption{The pipeline of our \ourname.}
    \label{alg:pipeline}
    \KwIn{Test samples $D_t=\{\bx_i\}$, source model $f_{{\bm \theta}^0}(\cdot)$, distribution $\mP$, learning rate $\eta_a$ and $\eta_p$.
    }
    
    Initialize $\mT = \{ \textbf{0} \}$, $\tilde{\bm \theta} =\textbf{0}$ and $\bm\alpha=\textbf{0}$. \\
    \For{$\bx \in D_t$} {
        // \textit{First forward pass.} \\
        Forward to obtain $\hat{\by}_i$ and compute $\mathcal{L}(\bx; \tilde{\bm \theta}, {\bm \alpha})$. \\
        Sample perturbations $\bnu$ and $\beps$ from distribution $\mP$. \\
        // \textit{Second forward pass.} \\
        Forward to compute $\mathcal{L}(\bx; \tilde{\bm \theta} + c {\bm \epsilon}, {\bm \alpha} + c {\bm\nu})$. \\
        Estimate gradients $\hat{g}(\tilde{\bm \theta})$  and $\hat{g}(\bm\alpha)$  based on \cref{eq:grad}. \\
        Update $\tilde{\bm \theta}$ using $\hat{g}(\tilde{\bm \theta})$ and $\eta_p$. \\
        Update $\bm\alpha$ using $\hat{g}(\balpha)$ and $\eta_a$. \\
        
        \If{Domain changes}{
        Store current parameters $\Delta_t$ into $\mT$ using \cref{eq:store}. \\
        Add new learnable parameters $\tilde{\bm \theta}'$ using \cref{eq:init_theta}.
        }
        \If{$|\mT| > N$}{
        Remove redundant parameters via \cref{eq:sim} and (\ref{eq:discard}).
        }
    }
    \KwOut{
    The predictions $\{ \hat{\by}_i \}$. 
    }
\end{algorithm}

\section{Continual Zeroth-Order Adaptation for QNNs}
In this paper, we propose a novel continual Zeroth-Order Adaptation (ZOA) framework, which is designed to efficiently adapt a quantized model with only two forward passes per sample to boost TTA practicality.
As shown in \cref{fig:overall}, ZOA comprises two key designs:
1)~A continual domain knowledge learning scheme based on zeroth-order optimization, which learns and reuses domain knowledge without forgetting for more effective TTA using two forward passes (c.f. \cref{sec:domain_learning}); and 2)~A domain knowledge management strategy that preserves diverse, and informative domain knowledge to reduce memory overhead while maintaining TTA performance (c.f. \cref{sec:domain_management}). We summarize the pseudo-code of ZOA in~\cref{alg:pipeline}.

\subsection{Continual Domain Knowledge Learning}
\label{sec:domain_learning}
Unlike conventional TTA methods that update the model via backpropagation, we aim to achieve both computationally and memory-efficient TTA—requiring only two forward passes per test sample. To this end, one direct solution is using the Covariance Matrix Adaptation Evolution Strategy~\cite{hansen2016cma} (CMA-ES), as in FOA~\cite{niu2024testtime}. However, this is inefficient and ineffective in our setting, due to: 1) CMA-ES relies solely on the rank of candidate solutions but ignores the detailed fitness (target loss) value, failing to provide sufficient learning signals when only two forward passes per sample are available, as in Fig. A of Appendix;  2) CMA-ES assumes a low-dimensional optimization space (\eg, prompt learning on ViT~\cite{niu2024testtime}), restricting its efficacy on quantized models with high-dimensional parameters.

In this paper, we seek to exploit the exact loss values of candidate solutions to deliver richer learning signals for more efficient and effective forward-only adaptation, and thus introduce a zeroth-order adaptation method. Additionally, we devise a continual learning scheme that accumulates learned knowledge without forgetting, thereby further boosting the efficiency and effectiveness of subsequent forward-only adaptations under a very limited number (\ie, 2) of forward passes. We depict them in the following.

\noindent \textbf{Zeroth-Order Gradient Estimation.} To efficiently update the QNN with two forward passes, we explore the zeroth-order optimization SPSA~\cite{spall1992multivariate} in TTA, which leverages the exact loss values $\mathcal{L}(\cdot)$ of solutions for gradient estimation. 
Specifically, the one-sided averaged gradient estimator of SPSA~\cite{liu2018zeroth, tsai2020transfer} can be formulated as:
\begin{equation} \label{eq:spsa}
\hat{g}({\bm \theta}) = \hat{\nabla} \mathcal{L}(\bx; {\bm \theta}) = \frac{1}{q} \sum^q_{i=1}
\frac{\mathcal{L}(\bx;{\bm \theta} + c {\bm \epsilon}_i) - \mathcal{L}(\bx;{\bm \theta})}{c} {\bm \epsilon}^{-1}_i,
\end{equation}
where $q$ denotes the step of perturbations, $ c > 0 $ is the perturbation scale, and ${\bm \epsilon}$ is a random perturbation vector sampled from mean-zero distributions, such as Rademacher and Segmented Uniform distribution~\cite{oh2023blackvip}.
For our ZOA, we set the number of steps $q$ to be 1. 
Therefore, we only require two forward passes to estimate the gradients, including one forward pass for standard inference and computing the loss of current parameters $\mathcal{L}(\bx;{\bm \theta})$, and one forward pass for computing the loss of perturbed parameters $\mathcal{L}(\bx;{\bm \theta} + c {\bm \epsilon})$.

\noindent\textbf{Continual Knowledge Reprogramming Learning.}
Although exploiting exact loss values in~\cref{eq:spsa} improves learning efficiency, using only two forward passes provides only a coarse gradient estimation, which may still limit the learning efficacy for TTA. On the other hand, increasing the number of forward passes to deliver more informative gradients would compromise the efficiency we aim to preserve. 
Alternatively, we explore the potential of continual learning with zeroth-order optimization, which seeks to accumulate and leverage the historical adaptation knowledge to boost the efficacy of subsequent TTA without expensive computational cost. 

To achieve knowledge accumulation, we first decouple the knowledge of different domains into domain knowledge vectors~\cite{yadav2023ties}, \ie, $\Delta_j = {\bm \theta}_j - {\bm \theta}^0 $, where ${\bm \theta}_j$ is adapted parameters on the $j$-th domain, and ${\bm \theta}^0$ is the source pre-trained parameters. We then save $\Delta_j$ in knowledge base $\mathcal{T} = \{\Delta_j\}_{j=0}^N$ upon distribution shifts (as discussed in~\cref{sec:domain_management}). Given the knowledge base $\mathcal{T}$, our continual knowledge reprogramming learning process is formulated as:
\begin{equation} \label{eq:domain_learning}
\begin{gathered}
\min_{{\bm \alpha}, \tilde{\bm \theta}} \mathcal{L}(\bx; \tilde{\bm \theta}),  ~ \text{where} ~ 
{\bm \theta} = {\bm \theta}^0 + \sum_{j=0}^N \alpha_j \Delta_j + \tilde{\bm \theta}, ~ \Delta_j \in \mathcal{T}.
\end{gathered}
\end{equation}
Here, $\tilde{\bm \theta}$ denotes a set of new learnable parameters to enable learning new domain knowledge. 
${\bm \alpha}$ is the learnable aggregation weights for different domain knowledge, normalized with the softmax function, \ie, $\sum_{j=0}^N \alpha_j = 1$. The relevant domain knowledge is retrieved to enhance adaptation with a large coefficient ${\alpha}_j$, while the irrelevant domain knowledge is suppressed with a small coefficient.

Within our zeroth-order optimization framework, we only update the parameters $\tilde{\bm \theta}$ and $\bm \alpha$, while freezing other parameters.
Notably, we perturb $\tilde{\bm \theta}$ and ${\bm \alpha}$ simultaneously at each step to avoid introducing an additional forward pass for~\cref{eq:spsa}. The estimated gradients of domain knowledge parameters $\tilde{\bm \theta}$ and aggragation coefficients ${\bm \alpha}$ are computed by:
\begin{equation} \label{eq:grad}
    \begin{bmatrix} \hat{g}(\tilde{\bm \theta}) \\ \hat{g}({\bm \alpha}) \end{bmatrix}= \frac{\mathcal{L}(\bx; \tilde{\bm \theta} + c {\bm \epsilon}, {\bm \alpha} + c {\bm \nu}) - \mathcal{L}(\bx; \tilde{\bm \theta}, {\bm \alpha})}{c} \begin{bmatrix} {\bm \epsilon}^{-1} \\ {\bm \nu}^{-1} \end{bmatrix},
\end{equation}
where ${\bm \epsilon}$ is the random perturbation vectors for $\tilde{\bm \theta}$ and $\bnu$ is the random perturbation vectors for ${\bm \alpha}$.
With the estimated gradients of learnable parameters, we can update the parameters using gradient descent methods, such as SGD or AdamW~\cite{loshchilov2018decoupled}.

\begin{table*}[t]
\caption{Effectiveness of our \ourname ~on \textbf{Quantized ViT-B models}.
We report the \textbf{Accuracy (\%)} on ImageNet-C (severity level 5) during the 10th round of continual adaptation.
``WNAN'' indicates that the ViT-B model is quantized to N-bit precision.}
\label{tab:quantized_vit}
\centering
\LARGE
\resizebox{1.0\linewidth}{!}{
\begin{tabular}{llccccccccccccccc>{\columncolor{black!8}}c}
\toprule
\multicolumn{1}{c}{} & \multicolumn{1}{c}{}& \multicolumn{3}{c}{Noise} & \multicolumn{4}{c}{Blur} & \multicolumn{4}{c}{Weather} & \multicolumn{4}{c}{Digital} & \multicolumn{1}{c}{Average} \\
\midrule
 Model & Method & Gauss. & Shot & Impul. & Defoc. & Glass & Motion & Zoom & Snow & Frost & Fog & Brit. & Contr. & Elas. & Pix. & JPEG & Acc.  \\   
\midrule
\multirow{4}{*}{ViT-B W8A8} & Source  &         56.1  & 56.0  & 56.7  & 46.7  & 34.9  & 52.5  & 42.6  & 61.1  & 61.7  & 65.9  & 77.1  & 24.3  & 44.5  & 65.6  & 66.8  & 54.2  \\  
& T3A~\cite{iwasawa2021test}  & 52.9  & 53.1  & 53.5  & 45.0  & 33.7  & 50.2  & 40.5  & 58.2  & 60.5  & 62.6  & 74.6  & 42.3  & 39.2  & 63.8  & 64.8  & 53.0  \\  
&  FOA~\cite{niu2024testtime} (K=2) & 57.1  & 57.5  & 58.0  & 49.0  & 37.8  & 54.1  & 45.6  & 63.8  & 64.7  & 70.5  & 77.5  & 54.0  & 48.6  & 66.4  & 67.8  & 58.2  \\
&  \ourname ~(ours)  &  60.2  & 62.0  & 61.8  & 54.6  & 50.2  & 59.1  & 55.2  & 67.0  & 65.5  & 68.0  & 79.1  & 58.0  & 59.6  & 70.7  & 70.5  & \textbf{62.8} \\
\midrule
\multirow{4}{*}{ViT-B W6A6} &  Source  & 44.2  & 42.0  & 44.8  & 39.8  & 28.9  & 43.4  & 34.7  & 53.2  & 59.8  & 59.0  & 75.1  & 27.4  & 39.0  & 59.1  & 65.3  & 47.7  \\ 
& T3A~\cite{iwasawa2021test} & 37.1  & 36.6  & 37.9  & 28.0  & 25.1  & 36.6  & 27.7  & 45.2  & 54.6  & 52.8  & 69.8  & 19.7  & 31.5  & 56.4  & 61.5  & 41.4 \\ 
& FOA~\cite{niu2024testtime} (K=2)   & 46.4  & 45.4  & 47.0  & 42.5  & 33.1  & 46.4  & 39.0  & 57.3  & 63.0  & 66.1  & 75.7  & 35.9  & 44.2  & 60.3  & 66.2  & 51.3 \\
& \ourname ~(ours)  & 52.6  & 54.2  & 55.0  & 48.4  & 42.8  & 53.7  & 46.7  & 60.5  & 63.8  & 64.8  & 76.5  & 42.0  & 50.0  & 65.8  & 68.2  & \textbf{56.3} \\
\midrule
\multirow{4}{*}{ViT-B W4A4} &  Source & 46.0  & 45.5  & 46.2  & 41.4  & 28.8  & 44.3  & 37.3  & 50.7  & 57.5  & 62.0  & 74.5  & 26.0  & 39.6  & 56.9  & 63.8  & 48.0 \\ 
& T3A~\cite{iwasawa2021test} & 41.3  & 41.6  & 42.1  & 39.1  & 27.9  & 41.5  & 34.6  & 47.1  & 55.8  & 60.1  & 71.7  & 30.0  & 36.1  & 53.8  & 61.5  & 45.6 \\ 
& FOA~\cite{niu2024testtime} (K=2)   & 47.5  & 46.8  & 47.6  & 43.4  & 31.2  & 46.8  & 40.4  & 54.8  & 59.6  & 64.7  & 74.8  & 39.5  & 43.4  & 57.7  & 64.8  & 50.9 \\
& \ourname ~(ours)  & 49.8  & 50.7  & 52.6  & 48.6  & 38.2  & 52.8  & 47.6  & 55.5  & 56.1  & 59.7  & 75.3  & 38.1  & 47.5  & 62.6  & 65.4  & \textbf{53.4} \\ 
\bottomrule
\end{tabular}
}
\end{table*}

\begin{table*}[t]
\centering
\caption{Effectiveness of our \ourname ~on \textbf{Quantized ViT-B} models in long-term continual adaptation. 
We report the average \textbf{Accuracy (\%)} on ImageNet-C (severity level 5) at each round of adaptation. 
``WNAN'' indicates that the ViT-B model is quantized to N-bit precision.
The \textbf{bold} number indicates the best result.
``\#FP'' is the number of forward passes to obtain output and update models.
}
\fontsize{6pt}{8pt}\selectfont
\resizebox{0.88\linewidth}{!}{
\setlength{\heavyrulewidth}{0.3pt}
\begin{tabular}{llccccccccccc>{\columncolor{black!8}}c}
\toprule
Models & Methods & \#FP & 1 & 2 & 3 & 4 & 5 & 6 & 7 & 8 & 9 & 10 & Average \\ 
\midrule
\multirow{4}{*}{ViT-B W8A8} & Source & 1 & 54.2  & 54.2  & 54.2  & 54.2  & 54.2  & 54.2  & 54.2  & 54.2  & 54.2  & 54.2  & 54.2  \\ 
& T3A~\cite{iwasawa2021test} & 1 & 55.6  & 55.6  & 55.1  & 54.9  & 54.6  & 54.3  & 53.9  & 53.7  & 53.4  & 53.0  & 54.4  \\ 
& FOA~\cite{niu2024testtime} (K=2) & 2 & 58.0  & 57.9  & 58.1  & 58.1  & 58.2  & 58.1  & 58.2  & 58.2  & 58.1  & 58.2  & 58.1 \\ 
& ZOA (ours) & 2 & 59.7  & 61.0  & 61.4  & 61.5  & 62.0  & 61.9  & 62.1  & 62.1  & 62.1  & 62.8  & \textbf{61.7} \\
\midrule
\multirow{4}{*}{ViT-B W6A6} & Source & 1 & 47.7  & 47.7  & 47.7  & 47.7  & 47.7  & 47.7  & 47.7  & 47.7  & 47.7  & 47.7  & 47.7  \\ 
& T3A~\cite{iwasawa2021test} & 1 & 43.3  & 44.7  & 44.3  & 43.9  & 43.5  & 43.0  & 42.6  & 42.2  & 41.9  & 41.4  & 43.1  \\ 
& FOA~\cite{niu2024testtime} (K=2) & 2 & 51.1  & 51.6  & 51.5  & 51.3  & 51.3  & 51.1  & 51.0  & 51.1  & 51.1  & 51.3  & 51.3  \\ 
& ZOA (ours) & 2 & 54.3  & 55.2  & 55.8  & 55.7  & 55.8  & 55.7  & 56.2  & 56.4  & 56.5  & 56.3  & \textbf{55.8} \\
\midrule
\multirow{4}{*}{ViT-B W4A4} & Source & 1 & 48.0  & 48.0  & 48.0  & 48.0  & 48.0  & 48.0  & 48.0  & 48.0  & 48.0  & 48.0  & 48.0  \\ 
& T3A~\cite{iwasawa2021test} & 1 & 47.3  & 47.6  & 47.3  & 47.0  & 46.6  & 46.4  & 46.2  & 46.0  & 45.9  & 45.6  & 46.6  \\ 
& FOA~\cite{niu2024testtime} (K=2) & 2 & 50.7  & 50.8  & 50.9  & 50.9  & 50.9  & 50.8  & 50.8  & 50.9  & 50.9  & 50.9  & 50.8  \\ 
& ZOA (ours) & 2 & 52.1  & 52.8  & 53.2  & 53.5  & 53.3  & 53.4  & 53.9  & 53.3  & 53.6  & 53.4  & \textbf{53.2} \\ 
\bottomrule
\end{tabular}
}
\label{tab:long_term}
\end{table*}

\noindent
\textbf{Objective Function for Quantized Model Adaptation.}
Our proposed continual zeroth-order optimization framework does not limit the objective function used to optimize $\balpha$ and $\tilde{\bm \theta}$. In our implementation, we select FOA~\cite{niu2024testtime} as our test-time learning objective, which provides stable learning signals for forward optimization with a feature discrepancy loss and an entropy minimization loss. Specifically, we first calculate the mean and standard deviations of the output activations of $L$ intermediate blocks $\{\mu_i^s, \sigma_i^s\}_{i=0}^{L}$ over a small set of unlabeled ID data.
During testing, we calculate the statistics $\{\mu_i(\bx_t), \sigma_i(\bx_t)\}_{i=0}^{L}$ over the current batch of test samples $\bx_t$.
The objective function for $B$ test samples $\bx_t$ is then given by:
\begin{equation} \label{eq:foa_loss}
\begin{aligned}
\mathcal{L}(\bx_t ; {\bm \theta}^t) &= \frac{1}{B \times C} \sum_{x \in \bx_t} \sum_{c \in \mathcal{C}} - y_c \log y_c \\
+ \frac{\lambda}{L} \sum_{i=1}^{L} & 
||\mu_i(\bx_t) - \mu_i^s||_2  + ||\sigma_i(\bx_t) - \sigma_i^s||_2,
\end{aligned}
\end{equation}
where ${\bm \theta}^t$ denotes the adapted parameters at $t$-th step.
$y_c$ is the c-th element of the prediction $\hat{\by}_t = f_{{\bm \theta}^t}(\bx_t)$, $C$ is the dimension of class space $\mathcal{C}$, and $\lambda$ is a trade-off parameter.
Note that a small number of unlabeled ID samples (\eg, 32 samples) is sufficient to compute the feature statistics of ID data~\cite{niu2024testtime}, making it practical for application.

\subsection{Domain Knowledge Management} 
\label{sec:domain_management}
In this section, we introduce our domain knowledge management strategy that preserves valuable learned knowledge during the continual TTA while maintaining a manageable memory footprint.

\noindent
\textbf{Domain Knowledge Preservation.}
In the long-term TTA, we use domain shift detection~\cite{hong2023mecta, chen2024crossdevice} to identify whether the test distribution changes. 
We compute the distribution distance of current test samples and historical samples, the domain change occurs when the distance is larger than a predefined threshold (refer to Appendix B for details).
Once changes occur, we accumulate the learned domain vectors in $\mathcal{T}$ and initialize a new set of parameters $\tilde{\bm \theta}$ in~\cref{eq:domain_learning} for further domain adaptation.
Specifically, the knowledge preservation process is formulated as:
\begin{equation}\label{eq:store}
\mT = \mT \cup \{\Delta_t \}, ~ ~~ \text{where} ~~ \Delta_t = {\bm \theta}_t - {\bm \theta}^0.
\end{equation}
Here, ${\bm \theta}_t$ is the ensemble parameters at the $t$-th domain from ~\cref{eq:domain_learning}. Note that the sum of all coefficients equals to 1 due to the softmax function, \ie, $\sum_{j=0}^n \alpha_j + \alpha_t = 1$, thus adding new domain vectors $\Delta_t$ in $\mathcal{T}$ can change the value of the ensemble parameter in~\cref{eq:domain_learning}.

To keep TTA stability, we seek to keep the ensemble parameters $\bm \theta_t$ unchanged after storing new domain vectors $ \Delta_t$ in $\mT$.
To this end, let $\bm{\theta}'$ be the updated ensemble parameters after adding new parameters in $\mT$, we re-initialize the $\tilde{\bm{\theta}}$ for new domain as:
\begin{equation} \label{eq:init_theta}
\tilde{\bm{\theta}}' = \tilde{\bm{\theta}} - (\bm{\theta}' - \bm{\theta}_t), 
\end{equation}
where $\tilde{\bm{\theta}}$ and $\tilde{\bm{\theta}}'$ denote the learnable parameters before and after vector preservation, respectively.
We also carefully initialize the coefficient $\alpha_t$ of the domain parameters $\Delta_t$ to keep the parameters $\tilde{\bm \theta}'$ initialized as small as possible, thereby reducing cross-domain interference. 
We put more details of $\balpha$ initialization in Appendix C.

\noindent
\textbf{Redundant Domain Knowledge Removal.}
Over time, the number of stored domain vectors in $\mT$ may grow excessively large, resulting in an unbearable memory footprint that hinders TTA practicality. 
To prevent this, we aim to preserve only the most diverse and informative ones to reduce the memory overhead.

To keep the memory footprint manageable, we restrict $\mathcal{T}$ to a maximum size $N$ and remove redundant domain vectors (\ie, those that are outdated and show a large similarity with other vectors) when the number of vectors in $\mathcal{T}$ exceeds $N$. Formally, the similarity between domain knowledge vectors is computed as:
\begin{equation} \label{eq:sim}
\begin{gathered}
k, p = \argmax_{\Delta_{i}, \Delta_{j} \in \mL} \cos(\Delta_{i}, \Delta_{j}), \\
\cos(\Delta_{i}, \Delta_{j}) = \frac{1}{L} \sum^L_{l=0} \frac{\langle \Delta_{i,l}, \Delta_{j,l} \rangle} {\| \Delta_{i,l}\| \|\Delta_{j,l}\|},
\end{gathered}
\end{equation}
where $k$ and $p$ are the indices of the selected domain parameters, $\Delta_{j,l}$ is the parameters in the $l$-th layer of $\Delta_{j}$, and $\langle\cdot, \cdot \rangle$ represents the inner product.
In implementation, we pre-compute the similarities between all $\Delta_j \in \mT$ to efficiently obtain the index $k$ and $p$.
Based on $k$ and $p$, we remove the most outdated domain vectors from $\mT$:
\begin{equation} \label{eq:discard}
\mT = \mT \setminus \{\Delta_{o} \},
\end{equation}
where $\setminus$ denotes the remove operation and $o = \min(k, p)$.

\begin{table*}[th]
\caption{Effectiveness of our \ourname ~on \textbf{W8A 8 ResNet-50} and \textbf{ViM-S} models.
We report the \textbf{Accuracy (\%)} on ImageNet-C (severity level 5) at the 10th round of adaptation.
The \textbf{bold} number indicates the best result.}
\label{tab:quantized_resnet_vim}
\centering
\LARGE
\resizebox{1.0\linewidth}{!}{
\begin{tabular}{llccccccccccccccc>{\columncolor{black!8}}c}
\toprule
\multicolumn{1}{c}{} & \multicolumn{1}{c}{}& \multicolumn{3}{c}{Noise} & \multicolumn{4}{c}{Blur} & \multicolumn{4}{c}{Weather} & \multicolumn{4}{c}{Digital} & \multicolumn{1}{c}{Average} \\
\midrule
Model & Method & Gauss. & Shot & Impul. & Defoc. & Glass & Motion & Zoom & Snow & Frost & Fog & Brit. & Contr. & Elas. & Pix. & JPEG & Acc.  \\
\midrule
\multirow{5}{*}{ResNet-50 W8A8} & Source  & 3.5  & 4.2  & 3.2  & 17.6  & 9.5  & 15.0  & 22.7  & 16.5  & 22.8  & 23.8  & 59.0  & 5.5  & 16.4  & 21.3  & 32.6  & 18.3   \\
& BN Adapt  & 16.1  & 17.3  & 16.9  & 14.4  & 15.1  & 25.8  & 38.4  & 34.2  & 33.2  & 47.6  & 65.3  & 15.9  & 44.4  & 49.1  & 40.1  & 31.6  \\  
& T3A~\cite{iwasawa2021test}  & 13.5  & 14.4  & 14.4  & 12.8  & 12.8  & 22.1  & 32.5  & 30.0  & 28.4  & 40.7  & 56.6  & 13.9  & 37.1  & 40.9  & 34.0  & 26.9 \\  
&  FOA~\cite{niu2024testtime} (K=2) & 15.9  & 17.4  & 15.3  & 12.8  & 13.3  & 20.2  & 30.2  & 25.7  & 27.4  & 37.0  & 56.9  & 12.0  & 36.2  & 43.2  & 37.8  & 26.7 \\   
&  \ourname ~(ours)  &  22.9  & 23.8  & 23.4  & 18.8  & 21.1  & 29.5  & 40.5  & 35.9  & 35.3  & 47.2  & 64.4  & 20.7  & 48.1  & 50.2  & 43.8  & \textbf{35.0} \\
\midrule
\multirow{4}{*}{ViM-S W8A8} &  Source  & 42.2  & 42.2  & 42.7  & 29.7  & 17.1  & 38.1  & 31.1  & 48.9  & 55.1  & 61.9  & 71.2  & 46.7  & 34.6  & 40.5  & 53.8  & 43.7 \\ 
& T3A~\cite{iwasawa2021test} &  39.3  & 39.3  & 39.2  & 29.4  & 16.7  & 36.9  & 30.0  & 48.2  & 53.9  & 59.0  & 70.3  & 49.2  & 34.1  & 39.4  & 53.3  & 42.6 \\ 
& FOA~\cite{niu2024testtime} (K=2)  &  42.8  & 42.0  & 44.1  & 27.6  & 15.2  & 37.1  & 32.2  & 49.3  & 55.7  & 60.5  & 70.4  & 51.6  & 35.6  & 40.1  & 53.7  & 43.9 \\
& \ourname ~(ours) & 46.5  & 47.9  & 48.7  & 31.4  & 23.8  & 43.5  & 36.2  & 50.7  & 55.0  & 63.5  & 71.6  & 49.6  & 37.0  & 43.4  & 55.6  & \textbf{47.0} \\ 
\bottomrule
\end{tabular}
}
\end{table*}

\begin{table*}[th]
\centering
\caption{Effectiveness of our \ourname ~on \textbf{Quantized ResNet-50} and \textbf{ViM-S} models in long-term continual adaptation. 
We report the average \textbf{Accuracy (\%)} on ImageNet-C (severity level 5) at each round of adaptation. 
Both models are quantized into 8-bit precision.
The \textbf{bold} number indicates the best result.
``\#FP'' is the number of forward passes to obtain output and update models.
}
\fontsize{6pt}{8pt}\selectfont
\resizebox{0.9\linewidth}{!}{
\setlength{\heavyrulewidth}{0.5pt}
\begin{tabular}{llccccccccccc>{\columncolor{black!8}}c}
\toprule
Models & Methods & \#FP & 1 & 2 & 3 & 4 & 5 & 6 & 7 & 8 & 9 & 10 & Average \\ 
\midrule
\multirow{6}{*}{ResNet50 W8A8} & Source & 1 & 18.3  & 18.3  & 18.3  & 18.3  & 18.3  & 18.3  & 18.3  & 18.3  & 18.3  & 18.3  & 18.3  \\ 
& BN Adapt & 1 & 31.6  & 31.6  & 31.6  & 31.6  & 31.6  & 31.6  & 31.6  & 31.6  & 31.6  & 31.6  & 31.6  \\ 
& T3A~\cite{iwasawa2021test} & 1 & 30.3  & 29.8  & 29.1  & 28.7  & 28.3  & 28.0  & 27.6  & 27.4  & 27.2  & 26.9  & 28.3  \\ 
& FOA~\cite{niu2024testtime} (K=2) & 2 & 24.9  & 26.2  & 26.2  & 26.4  & 26.5  & 26.7  & 26.7  & 26.8  & 26.8  & 26.7  & 26.4 \\ 
& ZOA (ours) & 2 & 32.8  & 33.7  & 34.1  & 34.1  & 34.4  & 34.6  & 34.3  & 34.4  & 34.6  & 35.0  & \textbf{34.2} \\
\midrule
\multirow{4}{*}{ViM-S W8A8} & Source & 1 & 43.7  & 43.7  & 43.7  & 43.7  & 43.7  & 43.7  & 43.7  & 43.7  & 43.7  & 43.7  & 43.7  \\ 
& T3A~\cite{iwasawa2021test} & 1 & 43.7  & 43.9  & 43.8  & 43.6  & 43.5  & 43.4  & 43.2  & 43.1  & 42.9  & 42.6  & 43.4  \\ 
& FOA~\cite{niu2024testtime} (K=2) & 2 & 43.9  & 43.9  & 43.9  & 43.9  & 43.9  & 43.9  & 43.9  & 43.9  & 43.9  & 43.9  & 43.9  \\ 
& \ourname ~(ours)  & 2 & 46.1  & 46.4  & 46.5  & 46.8  & 46.8  & 47.0  & 46.9  & 46.9  & 46.8  & 47.0  & \textbf{46.7} \\
\bottomrule
\end{tabular}
}
\label{tab:long_term_resnet_vim}
\end{table*}

\section{Experiment}

\noindent \textbf{Datasets and Models.}
We conduct experiments on the benchmark of test-time adaptation, such as ImageNet-C~\cite{hendrycks2018benchmarking}. 
ImageNet-C contains corrupted images in 15 types and each type has 5 severity levels. 
{In the experiments, we mainly report the results on ImageNet-C at severity level 5.
For simplicity, we abbreviate the 15 types of ImageNet-C as Gauss., Shot, Impul., Defoc., Glass, Motion, Zoom, Snow, Frost, Fog, Brit., Contr., Elas., Pix., and JPEG, respectively.}
We evaluate the different methods using ViT-B~\cite{dosovitskiy2021an}, ViM-S~\cite{zhu2024vision}, and ResNet-50~\cite{he2016deep}. 
These models are trained on the ImageNet-1K~\cite{deng2009imagenet} training set.
We obtain the model weights from the $timm$ repository~\cite{rw2019timm}.
To obtain the quantized models, we adopt PTQ4ViT~\cite{yuan2022ptq4vit} to obtain the W8A8 and W6A6 ViT-B models.
And we use AdaLog~\cite{wu2024adalog} to obtain the W4A4 and W3A3 ViT-B models.
For the ViM-S model, we adopt PTQ4VM~\cite{cho2025ptq4vm} to quantize the model into W8A8 precision.
Moreover, we use MQBench~\cite{li2021mqbench} to obtain the quantized W8A8 and W2A4 ResNet-50 models.
In the experiment, we use 'Source' to represent the deployed models without adaptation.
We put more details in Appendix D.

\noindent
\textbf{Compared Methods.}
In the experiment, we compare our \ourname ~with there BP-free methods, including LAME~\cite{boudiaf2022parameter}, T3A~\cite{iwasawa2021test}, and FOA~\cite{niu2024testtime}. 
Since the edge devices are resource-limited, we mainly compare FOA (K=2) in our experiments, which only uses two forward passes.
We also compare our methods with BP-based methods on the full-precision ViT-B model to show the effectiveness of our method on resource-abundant devices.
The compared BP-based methods include TENT~\cite{wang2021tent}, 
MEMO~\cite{zhang2022memo},
SAR~\cite{niu2023towards},  CoTTA~\cite{wang2022continual}, 
EATA~\cite{niu2022efficient},
and DeYO~\cite{lee2024entropy}.
We report the adaptation results on the test dataset in terms of Top-1 Accuracy {(abbreviated as Acc.)}.

\noindent
\textbf{Implementation Details.}
For fair comparison, we set the batch size of test samples to 64 as TENT~\cite{wang2021tent} and SAR~\cite{niu2023towards}.
Following FOA~\cite{niu2024testtime}, we use the validation set of ImageNet-1K to compute the statistics of ID data. 
During adaptation, we use the SGD optimizer to update the learnable domain parameters $\tilde{\bm \theta}$, and use the AdamW~\cite{loshchilov2018decoupled} to update the aggregation coefficients $\balpha$.
For all quantized models, we sample the perturbation vectors $\bepsilon$ and $\bnu$ from a Rademacher and Segmented Uniform distribution~\cite{oh2023blackvip}.
For the W8A8 ViT-B model, we add the perturbation vectors $\bepsilon$ and $\bnu$ with the step size of 0.02 and 0.05, respectively.
We set the learning rate of $\tilde{\bm \theta}$ and $\bm \alpha$ to be $0.0005$ and $0.01$, respectively.
We set the maximum number of domain knowledge parameters as $N=32$ to avoid large memory consumption. 
We put more details in Appendix D.

\begin{table}[t]
    \centering
    \caption{Comparisons w.r.t. computational complexity. FP/BP is short for forward/backward propagation. 
    Acc. (\%) is the average accuracy on ImageNet-C (level 5) at the 10th round of continual adaptation. 
    Memory (MB) is measured for processing 64 images on a single GPU.
    ``*'' denotes resetting the model parameters after adaptation on each corruption.
    }
    \LARGE
    \resizebox{1.0\linewidth}{!}{
    \begin{tabular}{c|cccc|cc}
    \toprule
        Model & \multicolumn{1}{c}{Method} & BP & \#FP & \#BP & \multicolumn{1}{c}{Acc.} & \multicolumn{1}{c}{Memory (MB)} \\ \hline
        \multirow{6}{*}{ViT-B} & Source & \ding{55} & 1 & 0 & 55.5  & 819 \\ 
        ~ & MEMO~\cite{zhang2022memo} & \ding{52} & 65 & 64 & 57.2  & 11,058 \\
        ~ & TENT*~\cite{wang2021tent}& \ding{52} & 1 & 1 & 59.6  & 5,165 \\ 
        ~ & TENT~\cite{wang2021tent} & \ding{52} & 1 & 1 & 0.1 & 5,165 \\ 
        ~ & CoTTA*~\cite{wang2022continual} & \ding{52} & 3 or 35 & 1 & 61.7  & 16,836 \\ 
        ~ & CoTTA~\cite{wang2022continual} & \ding{52} & 3 or 35 & 1 & 34.0  & 16,836 \\ 
        ~ & LAME~\cite{boudiaf2022parameter} & \ding{55} & 1 & 0 & 54.1  & 819 \\ 
        \midrule
        \multirow{4}{*}{ViT-B W8A8} & Source & \ding{55} & 1 & 0 & 54.2  & 205  \\ 
        ~ & T3A~\cite{iwasawa2021test} & \ding{55} & 1 & 0 & 53.0  & 239  \\ 
        ~ & FOA~\cite{niu2024testtime} (K=2) & \ding{55} & 2 & 0 & 58.2  & 208  \\ 
        ~ & \ourname (ours) & \ding{55} & 2 & 0 & \textbf{62.8}  & \textbf{207} \\ 
        \midrule
        \multirow{4}{*}{ViT-B W4A4} & Source & \ding{55} & 1 & 0 & 47.7  & 102  \\ 
        ~ & T3A~\cite{iwasawa2021test} & \ding{55} & 1 & 0 & 45.6  &  120 \\ 
        ~ & FOA~\cite{niu2024testtime} (K=2) & \ding{55} & 2 & 0 & 50.9  & 104  \\ 
        ~ & \ourname (ours) & \ding{55} & 2 & 0 & \textbf{53.4}  & \textbf{103} \\ 
        \midrule
        \multirow{4}{*}{ResNet-50 W8A8} & Source & \ding{55} & 1 & 0 & 18.3 & 195  \\
        ~ & T3A~\cite{iwasawa2021test} & \ding{55} & 1 & 0 & 26.9  & 249  \\ 
        ~ & FOA~\cite{niu2024testtime} (K=2) & \ding{55} & 2 & 0 & 26.7  &  195 \\ 
        ~ & \ourname (ours) & \ding{55} & 2 & 0 & \textbf{35.0} & \textbf{197} \\ 
        \bottomrule
    \end{tabular}
    }
    \label{tab:cost}
\end{table}

\begin{table}[!ht]
\centering
\caption{Ablation studies of each component on ImageNet-C (level 5).
ZO denotes we use the zeroth-order optimizer for adaptation, and DRL denotes our continual domain reprogramming learning.
We report the average accuracy at the first (R1) and the last round (R10) of continual adaptation.
}
\fontsize{8pt}{10pt}\selectfont
\resizebox{1.0\linewidth}{!}{
\setlength{\heavyrulewidth}{0.5pt}
\begin{tabular}{cccc|ccc}
\toprule
    \multirow{1}{*}{Method} & \multirow{1}{*}{$\#$FP} & \multirow{1}{*}{ZO} & \multirow{1}{*}{DRL} & \multirow{1}{*}{R1} & \multirow{1}{*}{R10} & \multirow{1}{*}{Memory (MB)} \\ 
    \midrule
    Source & 1 & ~ & ~ & 54.2 & 54.2 & 205 \\ 
    FOA & 2 & ~ & ~ & 58.0  & 58.2  & 208 \\ 
    V1 & 2 & \ding{52} & ~ & 59.7 & 60.2 & 206 \\ 
    ZOA (Ours) & 2 & \ding{52} & \ding{52} & \textbf{59.7} & \textbf{62.8} & 207 \\ 
    \bottomrule
\end{tabular}
}
\label{tab:ablation}
\end{table}

\begin{table}[th]
\centering
\caption{Effect of the number of domain parameters at the 10th round adaptation on ImageNet-C}
\resizebox{1.0\linewidth}{!}{
\begin{tabular}{c|cccccccc}
\toprule
N & 0 & 2 & 4 & 8 & 16 & 32 & 64 & 128 \\ \midrule
Accuracy (\%) & 33.6  & 33.5  & 34.1  & 34.6  & 34.9  & 35.0  & 35.4  & 35.3  \\
Memory (MB) & 0.00  & 0.02  & 0.03  & 0.07  & 0.13  & 0.27  & 0.53  & 1.06 \\ \bottomrule
\end{tabular}
}
\label{tab:effect_n}
\end{table}

\subsection{Results on Quantized Neural Networks}

\noindent \textbf{Results on Quantized ViT-B models.}
We evaluate our \ourname ~on the quantized ViT-B models.
Since the edge devices are resource-limited, we configure FOA to employ two forward passes as our \ourname.
As shown in \cref{tab:long_term}, our \ourname~ achieves the best performance in the first round of adaptation on ImageNet-C dataset.
Moreover, during the long-term adaptation process, our \ourname ~is able to accumulate the learned domain knowledge and further boost the adaptation performance.
We also show the results at the 10th round of adaptation in \cref{tab:quantized_vit}.
As shown in \cref{tab:quantized_vit}, our methods achieve a 5.0$\%$ improvement over FOA on average on the quantized W6A6 model.
These results on quantized models with different bit precisions demonstrate the superiority of our \ourname over existing methods.

\noindent \textbf{Results on Quantized ResNet-50.}
We evaluate our \ourname ~on the classical convolutional neural network, \ie, the quantized W8A8 ResNet-50 model in \cref{tab:quantized_resnet_vim} and \cref{tab:long_term_resnet_vim}.
From \cref{tab:quantized_resnet_vim}, FOA achieves limited results compared with the BN adapt baseline with only two forward passes.
This limitation arises because FOA relies on low-dimensional prompt learning to adapt the model, which is less flexible on CNN-based architectures.
Instead, our \ourname ~consistently adapts the quantized ResNet-50 to the target domains.
As shown in \cref{tab:long_term_resnet_vim}, our \ourname ~benefits from our continual knowledge accumulation to achieve better performance in long-term adaptation.

\noindent \textbf{Results on Quantized ViM-S.}
We also evaluate the different BP-free adaptation methods on the quantized Vision Mamba model, \ie, the quantized W8A8 ViM-S model from PTQ4VM~\cite{cho2025ptq4vm}.
Since PTQ4VM~\cite{cho2025ptq4vm} uses the per-token quantization scheme, it requires appropriate quantization parameters for activations related to the inserted prompts by FOA. 
However, obtaining suitable parameters for these new prompts is challenging during testing. This mismatch in quantization parameters can degrade the performance of the pre-quantized model, further affecting the adaptation capability of FOA.
As shown in \cref{tab:quantized_resnet_vim}, FOA achieves a limited performance gain compared with the source model. 
Instead, without requiring new prompt tokens, our \ourname ~does not suffer from the quantization scheme.
Thus, our method is still able to adapt the quantized models and achieve the best results. 
From \cref{tab:quantized_resnet_vim}, our method achieves a $3.1\%$ improvement over FOA on the quantized W8A8 ViM-S model.
The results on three quantized models demonstrate the effectiveness of our method on different architectures.

\noindent \textbf{Computational Complexity Analyses.}
In this part, we analyze the computational cost of different adaptation methods in ~\cref{tab:cost}.
Following ~\cite{liu2021post, niu2024testtime}, the memory consumption for quantized W8A8 and W4A4 models is an ideal estimation by 0.25$\times$  and 0.125$\times$ memory of 32-bit models, respectively.
For the adaptation of the W8A8 ViT-B model, our \ourname ~achieves the best results with merely two forward passes, whose computational cost is only approximately twice that of the standard forward of the quantized model.
The memory consumption of our \ourname ~is almost the same as that of the quantized model inference.
On the quantized W8A8 ViT-B model, our ~\ourname ~even achieves a better performance than the BP-based methods on full-precision models, such as MEMO~\cite{zhang2022memo}, TENT~\cite{wang2021tent}, and CoTTA~\cite{wang2022continual}.
Moreover, for the adaptation of the quantized ViT-B model with the W4A4  precision, our \ourname ~is able to adapt the model with merely 103 MB memory consumption.
In conclusion, the cost of our ~\ourname ~ is affordable for many resource-limited devices, such as smartphones, FPGAs, and laptops.

\subsection{Further Experiments}
In this part, we evaluate the component of our \ourname ~and provide more results and discussions in the following. Due to the page limit, we put more results and discussions in Appendix E.

\noindent \textbf{Ablation Studies.}
In this part, we analyze the effectiveness of each component of our \ourname.
We construct an additional baseline, namely V1, which does not use our continual domain knowledge learning scheme.
To show the effectiveness of our method, we compare these methods with the state-of-the-art FOA in terms of accuracy at the 10th round of adaptation.
As shown in \cref{tab:ablation}, our baseline V1 using zero-order adaptation already achieves better results than FOA on ImageNet-C. 
Moreover, the results of the last row show that our continual domain knowledge learning greatly boost the long-term adaptation for quantized models with merely 2 MB extra memory consumption compared with the source model. 

\noindent
\textbf{Effect of the number of domain parameters $N$.}
To investigate the hyperparameters N, we evaluate our ZOA on the W8A8 ResNet50 model. As shown in \cref{tab:effect_n}, when $N \ge 4$, our ZOA achieves promising improvements compared to the baseline without our continual knowledge learning scheme (\ie, $N=0$). In our experiments, we set N=32 for all quantized models by default.

\section{Conclusion}
In this paper, we theoretically and empirically illustrate the high sensitivity of quantized models to out-of-distribution data.
To this end, we propose a continual zeroth-order adaptation framework (\ourname) for efficient quantized model adaptation with only two forward passes per test data.
To further enhance the efficacy of forward-only TTA, we propose to accumulate and leverage the historical TTA knowledge with a continual knowledge reprogramming learning scheme, and devise a domain knowledge management strategy to reduce memory overhead.
Experiments on quantized models across different architectures demonstrate the effectiveness of our framework for efficient quantized model adaptation.

\begin{acks}
This work was partially supported by the National Natural Science Foundation of China (Grant No.U24A20327, U23B2013 and 62276176), Key-Area Research and Development Program of Guangdong Province (2018B010107001), TCL Science and Technology Innovation Fund, and the Young Scholar Project of Pazhou Lab (No.PZL2021KF0021).
\end{acks}

{
    \bibliographystyle{ACM-Reference-Format}
    \balance
    \bibliography{main}
}

\clearpage
\appendix
\setcounter{page}{1}
\setcounter{equation}{0}
\setcounter{figure}{0}
\setcounter{ass}{0}
\setcounter{prop}{0}
\renewcommand{\theequation}{\Alph{equation}}
\renewcommand{\thefigure}{\Alph{figure}}
\renewcommand{\theass}{\Alph{ass}}
\renewcommand{\theprop}{\Alph{prop}}

\twocolumn[{
  \begin{center}
    {\Huge\bfseries Appendix}
  \end{center}
  \vspace{1em}
}]

\section{Sensitivity Analysis of QNNs on OOD Data} \label{sec:proof}
In this part, we provide the details of the theoretical analysis of the loss sensitivity of quantized neural networks (QNNs) on out-of-distribution (OOD) data.

Let integer $n > 0$ denote the precision of quantized models. 
Let $\Delta \bW = \bW - \hat{\bW}$ represent the quantization error caused by the uniform quantization algorithm, where $\bW \in \mmR^{h \times w}$ and $\hat{\bW} \in \mmR ^ {h \times w}$ represent the parameters of the full-precision model and the quantized model, respectively.
Let $\bx$ be in-distribution (ID) data and $\bx+\delta$ be out-of-distribution data, where $\delta$ is the random perturbation caused by natural corruptions.

\begin{ass}
The ID data $\bx$ follows a gaussian distribution where $E[\bx] = 0$ and $Cov(\bx) =\Sigma_{\bx}> 0$.
Random perturbation $\bdelta$ is a zero-mean random perturbation where $E[\bdelta] = 0$ and $Cov(\bdelta) =\Sigma_{\bm \delta}> 0$, which is independent of ID data $\bx$.
The range of the parameters of the full-precision model is $[-a, a]$.
$\Delta \bW \neq \textbf{0}$ follows an i.i.d assumption and it is independent of $\bdelta$.
The linear model is formulated as $\by = \bW \bx$.
The ground-truth class of $\bx$ is ${\by}_g = \bW \bx + {\bm \bxi}$, where ${\bm \bxi}$ is zero-mean random noise of data generation, and ${\bm \bxi}$ independent of $\bx$, $\Delta \bW$ and $\delta$.
\end{ass}

\begin{prop} \label{prop:variance}
The expectation of quantization error $\Delta W_{ij} \in \Delta \bW$ is zero, and the variance of the quantization error is inversely proportional to the precision $n$, which is represented as:
\begin{equation}
Var(\Delta W_{ij}) = \frac{a^2}{3\cdot (2^{n} - 1)^2}.
\end{equation}
\end{prop}

\noindent
\begin{proof} 
The quantization interval is computed as follows:
\begin{equation}
\phi = \frac{2a}{2^n - 1}.
\end{equation}
The quantization process uses the rounding operation to convert the value between $[v-\phi/2, v+\phi/2]$ to the same value $v$ (such as $v=0$),
so the quantization error $ \Delta \bW$ follows a uniform distribution in $[-\phi/2, \phi/2]$, \ie, $\Delta \bW_{ij} \sim U(-\phi/2, \phi/2)$.
Thus, the expectation of the quantization error $\Delta \bW_{ij}$ is: 
\begin{equation} \label{eq:mean}
\mmE(\Delta \bW_{ij}) = \frac{-\phi/2 + \phi/2}{2} = 0.
\end{equation}
The variance of $\Delta \bW_{ij}$ is represented by:
\begin{equation}\label{eq:variance}
Var(\Delta \bW_{ij}) = \frac{(\phi/2 - (-\phi/2))^2}{12} = \frac{a^2}{3\cdot (2^{n} - 1)^2}.
\end{equation}
From \cref{eq:variance}, the variance of the
quantization error is inversely proportional to the precision n.
\end{proof}

\begin{prop} \label{prop:sensitivity_supp}
Considering linear models, let $n$ denote the bit-precision of quantized models, for out-of-distribution (OOD) input perturbations $\delta$,  the quantization-induced loss difference $\Delta \mathcal{L} := \hat{\mathcal{L}}(\mathbf{x}+\delta) - \mathcal{L}(\mathbf{x}+\delta)$ satisfies:
\begin{equation}
\Delta \mL >0, ~~ \text{and} ~~ \Delta \mL \propto \frac{1}{2^{2n}},
\end{equation}
where $\hat{\mathcal{L}}$ and $\mathcal{L}$ denote the MSE losses of quantized and full-precision models, respectively.
\end{prop}

\begin{proof}
The loss of full-precision models on OOD data is formulated as:
\begin{equation} \label{eq:fp_loss}
\begin{aligned}
\mL(\bx + \delta) & = \mmE[({\by}_g - \by)^2] \\ 
& = \mmE[((\bW \bx + \bxi) - \bW(\bx+ {\bm \delta}))^2]  \\ 
& = \mmE[(\bxi - \bW \bdelta)^2] \\
& = \mmE[\bxi^2 - 2 \bxi^\top \bW \bdelta + (\bW \bdelta)^2] \\
& = \mmE[\bxi^2] - 2 \mmE[\bxi^\top \bW \bdelta] + \mmE[\bdelta^\top \bW^\top \bW \bdelta] \\
& = Var(\bxi) + \mmE[Tr(\bW^\top \bW \bdelta \bdelta^\top)] \\
& = Var(\bxi) + Tr(\bW^\top \bW \Sigma_{\bm \delta}).
\end{aligned}
\end{equation}
where $Var(\bxi)$ is the variance of $\bxi$ and $\Sigma_{\bm \delta}$ is the covariance matrix of $\bdelta$.
The loss of quantized models on OOD data is formulated as:
\begin{equation} \label{eq:quant_loss}
\begin{aligned}
\hat{\mL}(\bx + \delta) &= \mmE[({\by}_g - \hat{\by})^2] \\
&= \mmE[((\bW\bx +  \bxi) - (\bW + \Delta \bW)(\bx + \bdelta))^2] \\
& = \mmE[(\bxi - \bW \bdelta - \Delta \bW \bx - \Delta \bW \bdelta)^2] \\
& = \mmE[\bxi^2 + (\bW\bdelta)^2 + (\Delta \bW\bx)^2 + (\Delta \bW \bdelta)^2 \\ 
& \quad - 2\bxi^\top \bW \bdelta - 2\bxi^\top \Delta \bW \bx - 2\bxi^\top \Delta \bW \bdelta \\
& \quad + 2 (\bW \bdelta)^\top \Delta \bW \bx + 2(\bW \bdelta)^\top \Delta \bW \bdelta + 2(\Delta \bW \bx)^\top \Delta \bW \bdelta] \\
& = \mmE[\bxi^2] + \mmE[(\bW\bdelta)^2] + \mmE[(\Delta \bW\bx)^2] + \mmE[(\Delta \bW \bdelta)^2] \\ 
& \quad - 2 \mmE[\bxi^\top \bW \bdelta] - 2 \mmE[\bxi^\top \Delta \bW \bx] - 2 \mmE[\bxi^\top \Delta \bW \bdelta] \\
& \quad + 2 \mmE[\bdelta^\top \bW^\top \Delta \bW \bx] + 2 \mmE[\bdelta^\top \bW^\top \Delta \bW \bdelta] \\
& \quad + 2 \mmE[\bx^\top \Delta \bW^\top \Delta \bW \bdelta] \\
&= Var(\bxi) + Tr(\bW^\top \bW \Sigma_{\bm \delta}) + Tr(\Delta \bW^\top \Delta\bW (\Sigma_{\bm x} + \Sigma_{\bm \delta})) \\
& \quad - 2 Tr(\bW \mmE[\bdelta] \mmE[\bxi^\top]) - 2\mmE[\bxi^\top] \mmE[\Delta \bW \bx] \\
& \quad - 2 \mmE[\bxi^\top] \mmE[\Delta \bW \bdelta] + 2 \mmE[\delta^\top] \mmE[\bW \Delta \bW \bx] \\
& \quad + 2 Tr(\bW^\top \mmE[\Delta \bW] \mmE[\bdelta \bdelta^\top]) + 2 \mmE[\bx^\top \Delta \bW^\top \Delta \bW] \mmE[\bdelta] \\
& = Var(\bxi) + Tr(\bW^\top \bW \Sigma_{\bm \delta}) + Tr(\Delta \bW^\top \Delta\bW (\Sigma_{\bm x} + \Sigma_{\bm \delta})).
\end{aligned}
\end{equation}
From \cref{eq:fp_loss} and \cref{eq:quant_loss}, the extra loss caused by the quantization process is formulated as:
\begin{equation}\label{eq:extra_loss}
\begin{aligned}
\Delta\mL &= \hat{\mL}(\bx + \bdelta) - \mL(\bx + \bdelta) \\
& = Tr(\Delta\bW^\top \Delta\bW (\Sigma_{\bm x} + \Sigma_{\bm \delta})) \\
& = Tr(\Delta\bW^\top \Delta\bW) \cdot Tr(\Sigma_{\bm x} + \Sigma_{\bm \delta}) \\
& = h \cdot w \cdot Var(\Delta \bW_{ij}) \cdot Tr(\Sigma_{\bm x} + \Sigma_{\bm \delta}).
\end{aligned}
\end{equation}
When test samples are known, the $\Sigma_{\bm x}$ and $\Sigma_{\bm \delta}$ are fixed and can be seen as a positive constant. 
From Proposition \ref{prop:variance}, we put $Var(\Delta \bW_{ij})$ into \cref{eq:extra_loss}, which is formulated as:
\begin{equation}\label{eq:simplified_loss}
\begin{gathered}
\Delta\mL = \frac{h \cdot w \cdot a^2}{3\cdot (2^{n} - 1)^2} \cdot Tr(\Sigma_{\bm x} + \Sigma_{\bm \delta}) = \frac{C}{(2^{n} - 1)^2} > 0.
\end{gathered}
\end{equation}
From \cref{eq:simplified_loss}, we conclude that the extra loss is inversely proportional to precision $n$:
\begin{equation}
\Delta\mL = \frac{C}{(2^{n} - 1)^2}  \propto \frac{1}{2^{2n}}.
\end{equation}
\end{proof}

From Proposition \ref{prop:sensitivity_supp}, the quantization models introduce an extra loss on out-of-distribution data, which is inversely proportional to their data precision n.
Therefore, the quantized models often performs worse on OOD data than their full-precision models, which shows the necessity to adapt the quantized models during test-time.

\begin{figure}
\centering
\includegraphics[width=1.0\linewidth]{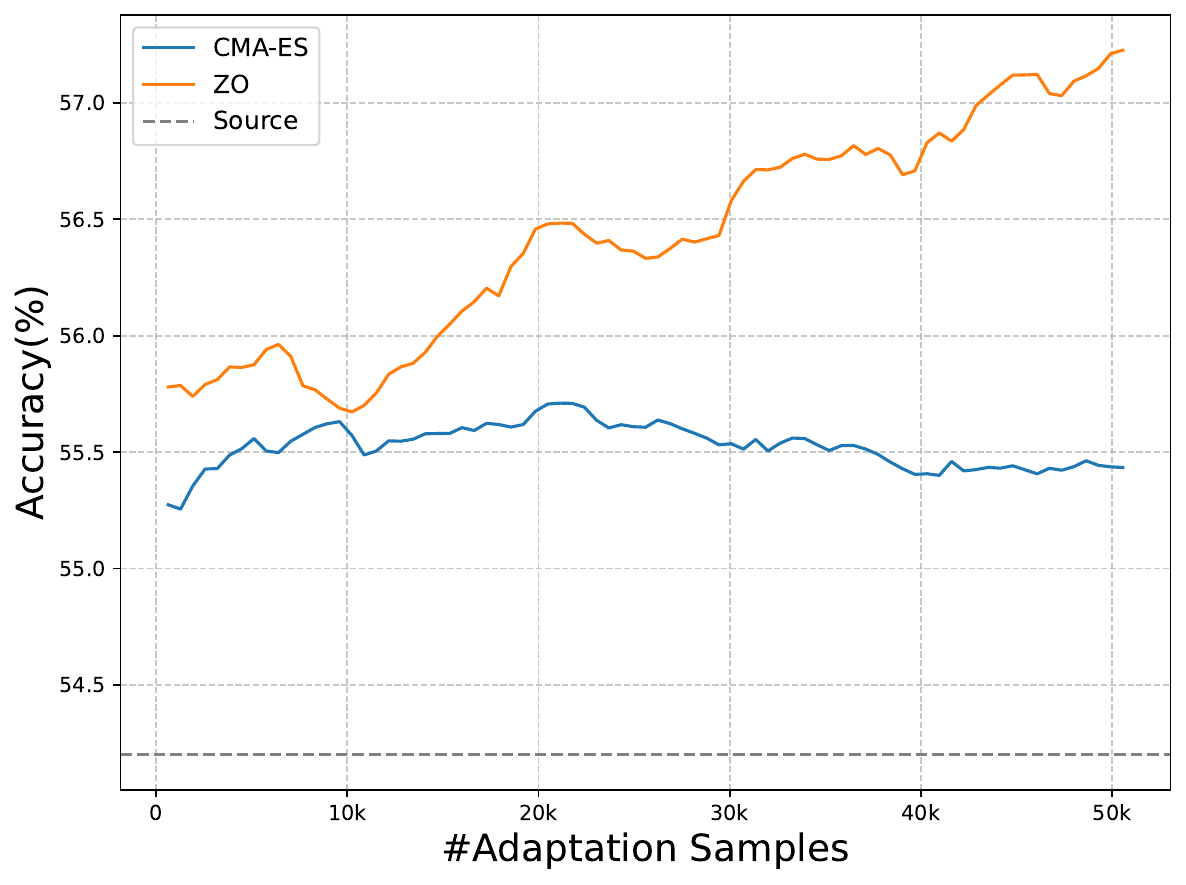}
\caption{
Online accuracy comparison between CMA-ES~\cite{hansen2003reducing} and Zeroth-Order (ZO) Optimizer~\cite{spall1992multivariate} on the quantized W8A8 ViT-B model and ImageNet-C (Gaussian noise, severity level 5). 
Results show that using the same number of test samples, ZO converges faster and achieves significantly better performance than CMA-ES.
}
\label{fig:cma_zo}
\end{figure}

\section{Domain Shift Detection} \label{sec:domain_detect}
During long-term test-time adaptation, it is crucial to detect the domain shift and aggregate the knowledge from different sets of domain knowledge parameters to boost the adaptation of quantized neural networks.
However, we often do not know any prior information about the domain labels for the test data stream.
Thus, it is required to detect the domain change to avoid interference between different domains during the long-term adaptation.
To this end, we utilize the domain shift detection scheme~\cite{hong2023mecta, chen2024crossdevice} to recognize the domain change.
Specifically, we first compute the distribution of the current domain using the statistics of the stem layer:
\begin{equation}\label{eq:domain_distribution}
\phi^{t-1}_d = \beta \phi_{t-1} + (1-\beta) \phi^{t-2}_d
\end{equation}
where $\phi_{t-1}$ comprises the mean and standard deviation calculated over the stem layer of the model $f_{{\bm \theta}^0}$ for the $(t-1)$-th batch of test samples.
$\beta=0.8$ is a moving average factor.
We then measure the distance between the current test samples and the current domain, which is formulated as:
\begin{equation}\label{eq:domain_distance}
\begin{gathered}
D(\phi_d^t, \phi_t) = \frac{1}{H}\sum_{i=1}^H KL(\phi^t_{d,i}|| \phi_{t,i}) + KL( \phi_{t,i}||\phi^t_{d,i}), \\
KL(\phi_1||\phi_2) = \frac{1}{2\sigma_2^2}(\sigma_1^2+(\mu_1-\mu_2)^2),
\end{gathered}
\end{equation}
where $H$ denotes the dimension of statistics.
Once the distance is larger than a predefined threshold $\tau$ (\eg, 0.1), the domain change occurs, \ie, the current domain shifts to another domain. 
At this moment, we store the domain knowledge parameters $\Delta_t$ and add a new set of learnable parameters $\tilde{\bm \theta}'$.

\begin{table*}[th]
\caption{Effectiveness of our \ourname ~on more quantized models, including ViT-L W8A8, ViT-B W3A3 and ResNet50 W2A4 models.
We report the \textbf{Accuracy (\%)} on ImageNet-C (severity level 5) during the 10th round of long-term continual adaptation.
``WNAN'' indicates that the weight and activation of deep models are quantized to N-bit precision, respectively.}
\label{tab:10round_more_qnn}
\centering
\LARGE
\resizebox{1.0\linewidth}{!}{
\begin{tabular}{llccccccccccccccc>{\columncolor{black!8}}c}
\toprule
\multicolumn{1}{c}{} & \multicolumn{1}{c}{}& \multicolumn{3}{c}{Noise} & \multicolumn{4}{c}{Blur} & \multicolumn{4}{c}{Weather} & \multicolumn{4}{c}{Digital} & \multicolumn{1}{c}{Average} \\
\midrule
 Model & Method & Gauss. & Shot & Impul. & Defoc. & Glass & Motion & Zoom & Snow & Frost & Fog & Brit. & Contr. & Elas. & Pix. & JPEG & Acc.  \\   
\midrule
\multirow{4}{*}{ViT-L W8A8} & Source  &         61.5  & 60.4  & 62.0  & 52.4  & 45.3  & 60.3  & 55.4  & 65.9  & 61.8  & 62.4  & 79.7  & 39.2  & 55.1  & 74.3  & 72.6  & 60.6  \\  
& T3A~\cite{iwasawa2021test}  & 53.3  & 52.5  & 53.4  & 46.8  & 43.2  & 56.0  & 51.9  & 61.9  & 59.5  & 57.9  & 76.2  & 36.8  & 52.9  & 70.2  & 68.0  & 56.0  \\  
&  FOA~\cite{niu2024testtime} (K=2) & 62.1  & 60.7  & 62.2  & 55.3  & 46.4  & 62.1  & 56.8  & 68.0  & 64.3  & 66.1  & 80.5  & 42.0  & 58.0  & 75.1  & 74.1  & 62.2  \\
&  \ourname ~(ours)  &  64.8  & 64.8  & 66.2  & 58.6  & 50.2  & 64.1  & 59.4  & 69.5  & 66.6  & 67.2  & 80.7  & 51.2  & 58.8  & 76.4  & 73.8  & \textbf{64.8} \\
\midrule
\multirow{4}{*}{ViT-B W3A3} &  Source  & 13.3  & 12.5  & 12.5  & 25.7  & 17.5  & 24.3  & 25.7  & 25.9  & 35.4  & 39.0  & 63.9  & 4.6  & 28.0  & 37.6  & 51.4  & 27.8  \\ 
& T3A~\cite{iwasawa2021test} & 12.3  & 11.5  & 11.2  & 23.2  & 16.7  & 22.3  & 23.7  & 23.5  & 35.0  & 35.3  & 61.4  & 4.4  & 26.7  & 34.3  & 48.4  & 26.0 \\ 
& FOA~\cite{niu2024testtime} (K=2)   & 14.6  & 14.0  & 14.1  & 27.1  & 18.7  & 27.4  & 28.1  & 31.6  & 37.4  & 43.3  & 64.1  & 8.2  & 31.0  & 39.4  & 52.5  & 30.1 \\
& \ourname ~(ours)  & 20.8  & 21.5  & 24.6  & 30.6  & 25.1  & 32.6  & 33.5  & 34.7  & 40.1  & 43.4  & 65.5  & 8.3  & 34.3  & 43.3  & 54.0  & \textbf{34.2} \\
\midrule
\multirow{4}{*}{ResNet50 W2A4} &  Source & 0.4  & 0.6  & 0.4  & 2.9  & 2.9  & 3.3  & 6.3  & 1.2  & 2.4  & 2.3  & 16.1  & 0.3  & 6.1  & 5.4  & 8.7  & 3.9 \\ 
& BN Adapt  & 4.7  & 5.1  & 4.9  & 5.9  & 5.8  & 10.0  & 16.6  & 12.0  & 11.2  & 19.0  & 38.6  & 4.3  & 22.0  & 20.8  & 18.2  & 13.3 \\
& T3A~\cite{iwasawa2021test} & 8.6  & 7.9  & 9.4  & 7.2  & 7.8  & 13.6  & 21.5  & 19.2  & 16.9  & 28.5  & 41.5  & 6.3  & 25.9  & 26.3  & 25.7  & 17.8 \\ 
& FOA~\cite{niu2024testtime} (K=2)   & 1.2  & 0.9  & 0.9  & 1.1  & 0.9  & 1.7  & 3.7  & 5.7  & 4.5  & 10.3  & 14.5  & 1.4  & 3.8  & 5.5  & 7.3  & 4.2 \\
& \ourname ~(ours)  & 13.2  & 13.7  & 13.6  & 10.2  & 11.8  & 17.1  & 26.1  & 21.8  & 21.7  & 33.5  & 48.3  & 9.1  & 33.6  & 35.2  & 34.1  & \textbf{22.9} \\ 
\bottomrule
\end{tabular}
}
\end{table*}

\begin{table*}[th]
\centering
\caption{Effectiveness of our \ourname ~on more quantized models in long-term continual adaptation. 
We report the average \textbf{Accuracy (\%)} on ImageNet-C (severity level 5) at each round of adaptation. 
``WNAN'' indicates that the weight and activation of models are quantized to N-bit precision, respectively.
The \textbf{bold} number indicates the best result.
``\#FP'' is the number of forward passes to obtain output and update models.
}
\fontsize{6pt}{8pt}\selectfont
\resizebox{0.88\linewidth}{!}{
\setlength{\heavyrulewidth}{0.3pt}
\begin{tabular}{llccccccccccc>{\columncolor{black!8}}c}
\toprule
Models & Methods & \#FP & 1 & 2 & 3 & 4 & 5 & 6 & 7 & 8 & 9 & 10 & Average \\ 
\midrule
\multirow{4}{*}{ViT-L W8A8} & Source & 1 & 60.6  & 60.6  & 60.6  & 60.6  & 60.6  & 60.6  & 60.6  & 60.6  & 60.6  & 60.6  & 60.6  \\ 
& T3A~\cite{iwasawa2021test} & 1 & 58.3  & 58.1  & 57.8  & 57.5  & 57.2  & 56.9  & 56.8  & 56.6  & 56.4  & 56.0  & 57.2  \\ 
& FOA~\cite{niu2024testtime} (K=2) & 2 & 62.2  & 62.1  & 62.2  & 62.2  & 62.2  & 62.2  & 62.1  & 62.1  & 62.2  & 62.2  & 62.2 \\ 
& ZOA (ours) & 2 & 62.5  & 62.6  & 63.1  & 63.2  & 63.6  & 64.0  & 64.1  & 64.5  & 64.2  & 64.8  & \textbf{63.7} \\
\midrule
\multirow{4}{*}{ViT-B W3A3} & Source & 1 & 27.8  & 27.8  & 27.8  & 27.8  & 27.8  & 27.8  & 27.8  & 27.8  & 27.8  & 27.8  & 27.8  \\ 
& T3A~\cite{iwasawa2021test} & 1 & 27.0  & 27.3  & 27.0  & 26.9  & 26.8  & 26.7  & 26.6  & 26.4  & 26.3  & 26.0  & 26.7  \\ 
& FOA~\cite{niu2024testtime} (K=2) & 2 & 30.0  & 29.9  & 29.8  & 29.9  & 29.8  & 29.9  & 29.9  & 29.9  & 30.0  & 30.1  & 29.9  \\ 
& ZOA (ours) & 2 & 31.9  & 33.9  & 34.9  & 34.5  & 35.1  & 35.2  & 34.8  & 35.1  & 34.6  & 34.2  & \textbf{34.4} \\
\midrule
\multirow{4}{*}{ResNet50 W2A4} & Source & 1 & 3.9  & 3.9  & 3.9  & 3.9  & 3.9  & 3.9  & 3.9  & 3.9  & 3.9  & 3.9  & 3.9  \\ 
& BN Adapt & 1 & 13.3  & 13.3  & 13.3  & 13.3  & 13.3  & 13.3  & 13.3  & 13.3  & 13.3  & 13.3  & 13.3  \\ 
& T3A~\cite{iwasawa2021test} & 1 & 19.9  & 20.2  & 19.7  & 19.3  & 19.0  & 18.7  & 18.4  & 18.1  & 17.9  & 17.8  & 18.9  \\ 
& FOA~\cite{niu2024testtime} (K=2) & 2 & 5.2  & 5.8  & 6.1  & 5.3  & 4.0  & 4.3  & 4.4  & 4.2  & 4.3  & 4.2  & 4.8 \\ 
& ZOA (ours) & 2 & 21.6  & 22.1  & 22.4  & 22.6  & 22.7  & 22.8  & 22.9  & 22.9  & 22.8  & 22.9  & \textbf{22.6} \\ 
\bottomrule
\end{tabular}
}
\label{tab:long_term_more_qnn}
\end{table*}

\section{Initialization of New Aggregation Weight}
\label{sec:initialization}
In our zeroth-order adaptation framework, we use a set of coefficients $\balpha$ to aggregate the knowledge of different domains. 
To keep the magnitude of the united parameters $\bm \theta$ stable during adaptation, we use the softmax function to ensure the sum of all coefficients is 1, \ie, $\sum_{j=0}^N \alpha_j = 1$.
Besides, we introduce a scaling temperature $T=10$ to adjust the sharpness of $\balpha$ during adaptation, enable the focus of aggregation on the most relevant domain knowledge.
Thus, we formulate $\balpha$ as $\balpha = softmax(\bbeta \cdot T)$, where $\bbeta$ is the logit vectors.

Once we detect that the domain change occurs, we store the current parameters in $\mT$, which is formulated as:
\begin{equation}\label{eq:store_domain}
\mT = \mT \cup \{\Delta_t \}, ~ ~~ \text{where} ~~ \Delta_t = {\bm \theta}_t - {\bm \theta}^0.
\end{equation}
Here, ${\bm \theta}_t$ denotes the ensemble parameters of the current domain and ${\bm \theta}^0$ is the parameters of the source model.
Then, we initialize a new coefficient $\alpha_t$ before the softmax operation for the new domain learning as follows:
\begin{equation} \label{eq:init_alpha}
\alpha_t = T \cdot \beta_{t} = \ln{\left((s-1)\sum\limits_{j=0}^{n}e^{\beta_{j}T}\right)},~s=\max\{1, \frac{m}{w_m}\},
\end{equation}
where $n=|\mT|$ is the number of domain parameters stored in $\mT$.
$m =\textrm{max}\{\frac{1}{J}\sum_{j=1}^{J}|\Delta^{(,j)}|\}_{=1}^{L}$ denotes the maximum magnitude of parameters across different layers where $J$ is the number of parameters in each layer.
With the {$\alpha_t$}, the magnitude of the recalculated $\tilde{\bm \theta}'$ in Eq. (8) is not larger than a constrained value $w_m = 0.01$.

\section{More Implementation Details}
\label{sec:details}

\noindent
\textbf{Test Data.}
The \textbf{ImageNet-C} dataset encompasses 15 distinct corruption types of 4 main groups, including Gaussian noise, shot noise, impulse noise, defocus blur, glass blur, motion blur, zoom blur, snow, frost, fog, brightness, contrast, elastic transformation, pixelation, and JPEG compression. 
Each corruption type is characterized by 5 different levels of severity, with higher severity levels indicating a more severe distribution shift.
In the experiments, we abbreviate the 15 fields as Gauss., Shot, Impul., Defoc., Glass, Motion, Zoom, Snow, Frost, Fog, Brit., Contr., Elas., Pix., and JPEG, respectively.
In the experiments, we adapt the models for every $B=64$ test samples.
In long-term adaptation, we continually adapt the quantized model to 15 corruptions over 10 rounds, so there are 150 corruptions in total.
\textbf{ImageNet-R}~\cite{hendrycks2021many} contains 30,000 images featuring diverse artistic renditions of 200 ImageNet classes. These images are predominantly sourced from Flickr and filtered by Amazon MTurk annotators.
\textbf{ImageNet-Sketch}~\cite{wang2019learning} consists of 50,899 images represented as black and white sketches, encompassing 1000 ImageNet classes. Each class contains approximately 50 images.

\noindent
\textbf{Implementation Details of FOA Loss.}
Following FOA~\cite{niu2024testtime}, we compute the mean and variants of CLS tokens for ViT-B and ViM-S models. 
As for Resnet-50 without CLS tokens, we use the global average pooling operation to obtain the features of each block with a shape of $(M, C)$ and then compute the mean and standard deviations over the dimension of $M$.
Following FOA~\cite{niu2024testtime}, we use the validation set of ImageNet-1K to compute the statistics of ID data for quantized model adaptation.

\noindent
\textbf{Implementation Details of ViT-B.}
We sample the perturbation vectors $\bepsilon$ and $\bnu$ from a Rademacher and Segmented Uniform distribution~\cite{oh2023blackvip}.
We perturb the learnable parameters of ViT-B models with a step size of 0.02.
The step size of the coefficients of different domain parameters is set to 0.05.
With the estimated gradient, we use the SGD optimizer with a weight decay of 0.4 to update the parameters $\tilde{\bm \theta}$ and use the AdamW~\cite{loshchilov2018decoupled} optimizer with a weight decay of 0.1 to update the coefficients $\bm \alpha$. 
For the experiment on resource-abundant devices, we store the domain knowledge for every 30 batches of test samples.
For a fair comparison, we remove all the stored parameters to reset the overall parameters to the pre-trained parameters after the adaptation on each domain.
During adaptation for quantized ViT-B models, we fix the parameters of the LayerNorm layers of the first block and the last three blocks.
We set the maximum number of domain knowledge parameters as $N=32$ to avoid large memory consumption.
We set the learning rate of $\alpha$ to be $0.01$ for all quantized ViT-B models.
For the W8A8 ViT-B model, we set the learning rate of $\tilde{\bm \theta}$ be $0.0005$.
For the W6A6 ViT-B model, we set the learning rate of $\tilde{\bm \theta}$ be $0.0002$.
For both W4A4 and W3A3 models, we set the learning rate to be $0.00005$.
We set the regularization term $\lambda$ = 30 to balance feature distribution alignment and prediction consistency.

\noindent
\textbf{Implementation Details of ResNet-50.}
We sample the perturbation vectors $\bm\epsilon$ and $\bm\nu$ from a Rademacher and Segmented Uniform distribution~\cite{oh2023blackvip}.
For the quantized W8A8 ResNet-50 model, we configure the base learning rate as 0.0001 for parameter updates and perturb the parameters with a step size of $c=0.01$ during adaptation.
{For the W2A4 ResNet50 model, we set the learning rate to be 0.00005 with the step size of $c=0.01$.}
The aggregation coefficients  $\alpha$ (see Eq. (4)) are optimized with a learning rate of 0.01 and a perturbation step size of 0.05. 
With the estimated gradient, we use the SGD optimizer with a weight decay of 0.4 to update the parameters $\tilde{\bm\theta}$ and use the AdamW~\cite{loshchilov2018decoupled} optimizer with a weight decay of 0.1 to update the coefficients $\bm\alpha$. 
To balance feature alignment and prediction consistency for ResNet without CLS tokens like ViT-B and ViM-S models, we set the balance factor $\lambda=1$ in Eq. (6). 
During adaptation for quantized ViM-S models, we fix the parameters of all BatchNorm layers of the first two blocks and the last six blocks, as well as the third BatchNorm layers of each block and those in the downsampling blocks.

\noindent
\textbf{Implementation Details of ViM-S.} 
For the quantized ViM-S model, we set a base learning rate of 0.0005 for parameter updates and apply perturbations with a step size of $c=0.03$ during adaptation.
We sample the perturbation vectors $\bm\epsilon$ and $\bm\nu$ from a Rademacher and Segmented Uniform distribution~\cite{oh2023blackvip}.
The aggregation coefficients $\bm\alpha$ are optimized with a learning rate of 0.01 and a perturbation step size of 0.03 to balance domain knowledge fusion.
With the estimated gradient, we use the SGD optimizer with a weight decay of 0.4 to update the parameters $\tilde{\bm \theta}$ and use the AdamW~\cite{loshchilov2018decoupled} optimizer with a weight decay of 0.1 to update the coefficients $\bm\alpha$. 
To stabilize training, we configure the weight of feature alignment regularization term $\lambda=30$ to balance feature distribution alignment and prediction consistency.
During adaptation for quantized ViM-S models, we fix the parameters of the RMSNorm layers of the first two blocks and the last six blocks.

\section{More Results}
\label{sec:more_results}

\subsection{Effectiveness of Our DKM}
We further investigate the effectiveness of our domain knowledge management scheme in \cref{tab:domain_management}.
To mimic the long-term adaptation that the domain changes thousands of times, we store the domain knowledge parameters after encountering every $30$ batch of test samples in the experiment.
Without our domain knowledge management (DKM),
there is a total of 4,015 sets of parameters after long-term continual adaptation.
Instead, equipped with our DKM, our method achieves a comparable performance while only storing 32 sets of domain parameters. 
We also compare with two baselines.
1) \textbf{Dequeue} represents that we discard one set of domain parameters from the stored set $\mT$ with the smallest index.
2) \textbf{Random} represents that we randomly discard one set of domain parameters from $\mT$.
As shown in \cref{tab:domain_management}, our \ourname ~achieves higher adaptation performance than the dequeue and random discard operations.
Moreover, our ~\ourname ~also achieves a comparable adaptation performance compared with the ``save all'' scheme with only 0.8\% memory consumption.
To sum up, these experimental results demonstrate the effectiveness and high efficiency of our domain knowledge management scheme.

\begin{table}[!ht]
\centering
\caption{Comparisons of different methods to manage the stored domain knowledge parameters.
``Save all'' indicates the variant method that preserves all the domain knowledge parameters.
``Dequeue'' denotes the variant that discards one set of parameters from $\mT$ with the smallest index.
``Random'' is the variant that randomly discards one set of parameters.
}
\resizebox{1.0\linewidth}{!}{
\begin{tabular}{cccc}
\toprule
    Method & \#Domain Parameters & Memeory (MB) & Acc. \\ 
    \midrule
    Source & - & - & 54.2  \\ 
    Save All & 4,015 & 86.2  & 62.2 \\ 
    Dequeue & 32 & 0.7 & 50.2 \\
    Random & 32 & 0.7 & 60.3 \\
    DKM (Ours) & 32 & 0.7 & 61.3 \\ 
    \bottomrule
\end{tabular}
}
\label{tab:domain_management}
\end{table}



\subsection{Effect of Different Perturbation Scale}
We investigate the effect of different perturbation scales $c$ on the W8A8 ResNet50 model in \cref{tab:effect_c}. 
In this experiment, we keep the perturbation of $\bm\alpha$ unchanged, and only change the perturbation scale of the learnable $\tilde{\bm \theta}$.
Compared with BN Adapt ($31.6\%$) and FOA~\cite{niu2024testtime} ($26.4\%$), our ZOA remains effective in the range of $c \in [0.005, 0.03]$. In our experiments, we set $c=0.01$ for the W8A8 ResNet-50 model.

\begin{table}[th]
\centering
\caption{Effect of different perturbation scales $c$ at the 10th round adaptation on ImageNet-C}
\begin{tabular}{cccccc}
\toprule
c & 0.05 & 0.03 & 0.02 & 0.01 & 0.005 \\ 
\midrule
Accuracy (\%) & 30.7  & 33.6  & 34.6  & \textbf{35.0}  & 33.7 \\ 
\bottomrule
\end{tabular}
\label{tab:effect_c}
\end{table}

\subsection{Effect of Perturbation Distribution}
We compare the results of ZOA using Normal Gaussian (Gauss.) or Rademacher and Segmented Uniform (RSU) distributions for sampling perturbations. Results of \cref{tab:effect_dist} show that RSU works well on the W8A8 ViT-B model ($62.8\%$ vs. $61.6\%$), while Gauss. works well on the W8A8 ResNet-50 model ($35.8\%$ vs. $35.0\%$). For simplicity, we use RSU for all experiments without careful tuning.

\begin{table}[th]
\centering
\caption{Results of ZOA using Normal Gaussian (Gauss.) or Rademacher and Segmented Uniform (RSU) distributions for sampling perturbations.}
\begin{tabular}{ccc}
\toprule
Model & Gauss. & RSU \\ 
\midrule
ViT-B W8A8 & 61.6  & \textbf{62.8} \\ 
ResNet 50 W8A8 & \textbf{35.8}  & 35.0 \\ 
\bottomrule
\end{tabular}
\label{tab:effect_dist}
\end{table}

\subsection{Effect of Different Initialization of $\alpha_t$}
When encouter a new domain, we use \cref{eq:init_alpha} add a new $\alpha_t$ into $\boldsymbol \alpha$ for the new stored domain paramters $\Delta \boldsymbol{\theta}_t$. 
Notably, suboptimal initialization would indeed impede adaptation. If $\alpha_t$ is initialized to be very large ($\alpha_t \to 1$), the distribution of $\boldsymbol{\alpha}$ is extremely sharp, hindering effective knowledge selection from existing domains (35.0\% vs. 19.6\%). If $\alpha_t$ is too small ($\alpha_t \to 0$), the magnitude of the initialized $\Delta \tilde{\theta}'$ via Eq. (8) is almost the same as the learned $\Delta \tilde{\theta}$ of the previous domain, allowing irrelevant historical knowledge to distort current domain learning ($35.0\%$ vs. $29.9\%$). If $\alpha_t$ is randomly initialized, the interference between the previous domain and the current domain may still exist, leading to a larger variance in adaptation performance (35.0\% vs. 32.5\%).

\begin{table*}[!ht]
\caption{Comparisons with SOTA methods on ImageNet-C (severity level 5) with ViT-B regarding \textbf{Accuracy (\%)}. \textbf{BP} is short for \textbf{backward propagation} and the \textbf{bold} number indicates the result of our method.}
\label{tab:fp_vit}
\newcommand{\tabincell}[2]{\begin{tabular}{@{}#1@{}}#2\end{tabular}}
\centering
\LARGE
\resizebox{1.0\linewidth}{!}{
\begin{tabular}{lcccccccccccccccc>{\columncolor{black!8}}c}
\toprule
\multicolumn{1}{c}{} & \multicolumn{1}{c}{}& \multicolumn{3}{c}{Noise} & \multicolumn{4}{c}{Blur} & \multicolumn{4}{c}{Weather} & \multicolumn{4}{c}{Digital} & \multicolumn{1}{c}{Average} \\
\midrule
Method & BP & Gauss. & Shot & Impul. & Defoc. & Glass & Motion & Zoom & Snow & Frost & Fog & Brit. & Contr. & Elas. & Pix. & JPEG & Acc. \\

\midrule
Source & \ding{55} & 56.8  & 56.8  & 57.5  & 46.9  & 35.6  & 53.1  & 44.8  & 62.2  & 62.5  & 65.7  & 77.7  & 32.6  & 46.0  & 67.0  & 67.6  & 55.5   \\ 
LAME~\cite{boudiaf2022parameter} & \ding{55} & 56.5  & 56.5  & 57.2  & 46.4  & 34.7  & 52.7  & 44.2  & 58.4  & 61.5  & 63.1  & 77.4  & 24.7  & 44.6  & 66.6  & 67.2  & 54.1  \\
T3A~\cite{iwasawa2021test} & \ding{55} & 56.4  & 56.9  & 57.3  & 47.9  & 37.8  & 54.3  & 46.9  & 63.6  & 60.8  & 68.5  & 78.1  & 38.3  & 50.0  & 67.6  & 69.1  & 56.9    \\ 
FOA~\cite{niu2024testtime}  & \ding{55}    & 61.5  & 63.2  & 63.3  & 59.3  & 56.7  & 61.4  & 57.7  & 69.4  & 69.6  & 73.4  & 81.1  & 67.7  & 62.7  & 73.9  & 73.0  & 66.3 \\
TENT~\cite{wang2021tent} & \ding{52} & 60.3  & 61.6  & 61.8  & 59.2  & 56.5  & 63.5  & 59.2  & 54.3  & 64.5  & 2.3  & 79.1  & 67.4  & 61.5  & 72.5  & 70.6  & 59.6    \\ 
CoTTA~\cite{wang2022continual} & \ding{52} & 63.6 & 63.8 & 64.1 & 55.5 & 51.1 & 63.6 & 55.5 & 70.0 & 69.4 & 71.5 & 78.5 & 9.7 & 64.5 & 73.4 & 71.2 & 61.7 \\
EATA~\cite{niu2022efficient} & \ding{52} & 62.0  & 65.5  & 65.5  & 59.3  & 60.7  & 65.0  & 63.8  & 69.2  & 68.8  & 73.4  & 80.3  & 58.4  & 68.9  & 74.4  & 73.6  & 67.2 \\
SAR~\cite{niu2023towards}  & \ding{52}    & 59.2  & 60.5  & 60.7  & 57.5  & 55.6  & 61.8  & 57.6  & 65.9  & 63.5  & 69.1  & 78.7  & 45.7  & 62.4  & 71.9  & 70.3  & 62.7    \\ 
DeYO~\cite{lee2024entropy}  & \ding{52}    & 59.8  & 61.5  & 61.1  & 57.4  & 59.0  & 64.5  & 61.9  & 69.1  & 66.7  & 69.5  & 78.9  & 65.3  & 69.6  & 74.0  & 72.3  & 66.0   \\
\midrule
\ourname ~(Ours)  & \ding{55}    & 61.6  & 63.1  & 63.5  & 59.7  & 59.0  & 64.9  & 62.6  & 70.4  & 68.4  & 74.0  & 80.6  & 67.1  & 69.0  & 74.7  & 73.2  & \textbf{67.5} \\
\bottomrule
\end{tabular}
}
\end{table*}

\subsection{Results on More Quantized Models}
In this part, we further conduct additional experiments on the larger model (quantized ViT-L W8A8 model), the ViT-B model with lower bit precision (W3A3), and the ResNet50 model with hybrid precision (W2A4).
Using SPSA~\cite{spall1992multivariate} for gradient estimation, our ZOA does not limit the size of model parameters and can effectively scale to larger models.
As shown in \cref{tab:10round_more_qnn} and \cref{tab:long_term_more_qnn}, our \ourname ~outperforms the compared baselines by a large margin on the W8A8 ViT-L model, \ie, +4.2\% (Ours) vs. +1.6\% (FOA).
Moreover, our ZOA achieves a 4.1\% improvement over FOA (34.2\% vs. 30.1\%) on the W3A3 ViT-B model and a 5.1\% improvement over T3A (22.9\% vs. 17.8\%) on the W2A4 ResNet50 model, further demonstrating our effectiveness on quantized models with lower bit precision or hybrid precision.

\begin{table}[!th]
\centering
\caption{Results on ImageNet-R and ImageNet-Sketch.}
\label{tab:result_others}
\begin{tabular}{cccc}
\toprule
Model & Method & ImageNet-R & ImageNet-Sketch \\ 
\midrule
\multirow{5}{*}{ResNet50 W8A8} & Source & 36.2  & 24.0  \\ 
~ & BN Adapt & 39.6  & 26.3  \\ 
~ & T3A & 18.3  & 20.5  \\ 
~ & FOA (K=2) & 35.2  & 25.6  \\ 
~ & ZOA (ours) & \textbf{40.5}  & \textbf{27.1}  \\ 
\midrule
\multirow{4}{*}{ViT-B W8A8} & Source & 58.4  & 44.1  \\ 
~ & T3A & 30.8  & 35.5  \\ 
~ & FOA (K=2) & 58.9  & 46.1  \\ 
~ & ZOA (ours) & \textbf{62.6}  & \textbf{48.7} \\ 
\bottomrule
\end{tabular}
\end{table}

\subsection{Results on More Adaptation Datasets}
In this part, we further evaluate our methods on two common adaptation benchmarks, including ImageNet-R~\cite{hendrycks2021many} and ImageNet-Sketch~\cite{wang2019learning} datasets. 
In this experiment, we adapt the quantized model, including the W8A8 ViT-B model and the W8A8 ResNet50 model, for 10 rounds across these two datasets.
As shown in \cref{tab:result_others}, our ZOA consistently achieves better results, further suggesting our effectiveness across different adaptation datasets.

\subsection{Reasonability of Using a Limited Capacity}
A limited capacity N works well for test-time adaptation since: \textbf{1)} Relevant domains often share distribution similarities, making \textbf{learned knowledge transferable}. For instance, a ViT-B model adapted to Gaussian noise with TENT~\cite{wang2021tent} significantly improves the performance on Shot noise compared to the source model (62.8\% vs. 56.9\%). This mitigates the need to store domain-specific parameters for every domain. \textbf{2)} Our DKM measures similarity across domains and retains unique knowledge by only \textbf{discarding redundant ones that show high similarities}. This ensures that unique knowledge is preserved even under tight memory constraints.

\subsection{More Results of the 32FP ViT-B Model}
To explore the effectiveness of our \ourname ~on resource-abundant devices, we further evaluate our method with state-of-the-art BP-based methods. 
To this end, we use more forward passes as FOA~\cite{niu2024testtime} to achieve better results on the ViT-B 32FP model.
Besides, the parameters are reset to the pre-trained models after adaptation on each domain.
As shown in \cref{tab:fp_vit} of Appendix, our \ourname ~achieves an even better performance than many classical BP-based methods, such as TENT~\cite{wang2021tent}, SAR~\cite{niu2023towards}, CoTTA~\cite{wang2022continual}, and DeYO~\cite{lee2024entropy}.
These results demonstrate that our \ourname ~also has a great potential to be applied on resource-abundant devices or time-insensitive applications.
We also compare our ZOA with EATA in the continual TTA setting. As shown in \cref{tab:long_term_fp}, our ZOA achieves significantly better performance than EATA with much less memory consumption (only about 15\% memory consumption of EATA).

\begin{table}[!ht]
\centering
\caption{Comparisons with EATA using the 32FP ViT-B model in long-term adaptation on ImageNet-C.}
\label{tab:long_term_fp}
\begin{tabular}{ccccc}
\toprule
Method & BP & Memory & R1 & R10 \\ 
\midrule
Source & \ding{55} & 819 & 55.5 & 55.5 \\ 
EATA & \ding{52} & 5506 & 66.7  & 52.7  \\ 
ZOA (ours) & \ding{55} & 827 & \textbf{67.1}  & \textbf{67.7} \\ 
\bottomrule
\end{tabular}
\end{table}



\end{document}